\documentclass[twoside,11pt]{article}
\usepackage{paralist,amsmath, amssymb}
\usepackage{algorithm,algorithmic}
 \usepackage{graphicx}
 \usepackage{float}
\usepackage{subfigure} 
\usepackage{xcolor}
\usepackage{lipsum}
\usepackage{graphics}
\usepackage[sort, numbers]{natbib}
 \usepackage{amssymb, multirow, paralist, color}
 \usepackage{amsthm}
 \usepackage{textcase, booktabs,makecell}
\usepackage{enumitem}
\def \y {\mathbf{y}}
\def \E {\mathrm{E}}

\def \x {\mathbf{x}}

\def \bv {\mathbf{v}}

\def \D {\mathcal{D}}

\def \u {\mathbf{u}}

\def \w {\mathbf{w}}
\def \R {\mathbb{R}}
\def \m {\mathbf{m}}

\def \F {\mathcal{F}}

\def \B {\mathbf{B}}

\def \m {\mathbf{m}}

\def \T {\mathcal{T}}
\def \F {\mathcal{F}}

\def \B {\mathcal{B}}

\def \S {\mathcal{S}}
\makeatletter
\def\blfootnote{\gdef\@thefnmark{}\@footnotetext}
\makeatother

\newtheorem{ass}{Assumption}
\newtheorem{tho}{Theorem}
\newtheorem{lemma}{Lemma}

\newtheorem{Remark}{Remark}

\usepackage{microtype}
\usepackage{graphicx}
\usepackage{subfigure}
\usepackage{booktabs} 
\usepackage{jmlr2e}

\begin{document}
\title{Momentum Accelerates the Convergence of \\ Stochastic AUPRC Maximization}
\author{\name Guanghui Wang \email wanggh@lamda.nju.edu.cn\\
       \addr National Key Laboratory for Novel Software Technology,
       Nanjing University, Nanjing 210023, China\\
       \name Ming Yang \email yangming@mail.hfut.edu.cn\\
       \addr Hefei University of Technology, Hefei 230000, China\\
       \name Lijun Zhang \email zhanglj@lamda.nju.edu.cn\\
       \addr National Key Laboratory for Novel Software Technology, 
       Nanjing University, Nanjing 210023, China\\
       \name Tianbao Yang \email tianbao-yang@uiowa.edu\\
       \addr Department of Computer Science, the University of Iowa, Iowa City, IA 52242, USA
       }

\maketitle

\begin{abstract}
\noindent\ignorespacesafterend In this paper, we study stochastic optimization of areas under precision-recall curves (AUPRC), which is widely used for combating imbalanced classification tasks.  Although a few methods have been proposed for maximizing AUPRC, stochastic optimization of AUPRC with convergence guarantee remains an undeveloped territory. A state-of-the-art complexity is $O(1/\epsilon^5)$ for finding an $\epsilon$-stationary solution. In this paper, we further improve the stochastic optimization of AURPC by (i) developing novel stochastic momentum methods with a better iteration complexity of $O(1/\epsilon^4)$ for finding an $\epsilon$-stationary solution; and (ii) designing a novel family of stochastic adaptive methods with the same iteration complexity, which enjoy faster convergence in practice. To this end, we propose two innovative techniques  that are critical for improving the convergence: (i) the biased estimators for tracking individual ranking scores are updated in a randomized coordinate-wise manner; and (ii)  a momentum update is used on top of the stochastic gradient estimator for tracking the gradient of the objective. The novel analysis of Adam-style updates is also one main contribution.   Extensive experiments on various data sets demonstrate the effectiveness of the proposed algorithms. Of independent interest, the proposed stochastic momentum and adaptive algorithms are also applicable to a class of two-level stochastic dependent compositional optimization problems. 
\end{abstract}

\section{INTRODUCTION} 
In supervised \blfootnote{This work has been accepted by AISTATS'22.} machine learning systems, the performance metrics 
used for model evaluation play a vital role 
\citep{ferri2009experimental}. Traditional machine learning 
algorithms typically employ  \emph{accuracy} (proportion of correctly predicted examples) as the measure of performance, which is a natural choice when the classes of data are balanced. However, in many real-world applications, such as activity recognition \citep{gao2016adaptive}, and medical diagnosis \citep{krawczyk2016evolutionary}, the distribution of classes is highly skewed, for which the accuracy usually fails to characterize the hardness of the problem. In such cases, a more enhanced metric, named areas under the precision-recall curves (AUPRC), is proposed and commonly used to assess classifiers \citep{boyd2013area}. Over the past decades,  the superiority of AUPRC for \emph{evaluating} imbalanced classifiers has been witnessed by a long line of research \citep{davis2006relationship,clemenccon2009nonparametric,kok2010learning,boyd2012unachievable,flach2015precision}.


How to utilize AUPRC to \emph{facilitate algorithm design} remains a challenging  question. As observed by \cite{cortes2003auc}, maximizing accuracy  on training data does not necessarily lead to a satisfactory solution with maximal AUPRC. On the other hand, directly optimizing AUPRC is generally intractable due to the complicated integral operation. To mitigate this issue, most of the existing works seek to optimize certain estimator of AUPRC~\citep{qin2008a,DBLP:conf/eccv/BrownXKZ20,qi2021stochastic}. In this paper, we focus on maximizing \emph{average precision} (AP), which is  one of the most commonly used estimators in practice for the purpose of maximizing AUPRC. Given a training set $\D=\{(\x_i,y_i)\}_{i=1}^n$, AP is defined as \citep{boyd2013area}:
\begin{align}\label{eqn:AP}
\text{AP}=\frac{1}{m}\sum_{\x_i, y_i = 1}\frac{\text{r}^+(\x_i) }{\text{r}(\x_i)},
\end{align}
where $(\x_i, y_i=1)$ denotes a positive example, $\text{r}^+(\x_i)$ denotes its rank among all positive examples (i.e., the number of positive examples that are ranked higher than $\x_i$ including itself),  $\text{r}(\x_i)$ denotes its rank among all  examples (i.e., the number of examples that are ranked higher than $\x_i$ including itself), and $m$ denotes the total number of positive examples. It can be proved that AP is an unbiased estimator of AUPRC as the number of examples $n$ goes infinity. 

However, the optimization of AP is  challenging  due to the non-differential ranking functions $\text{r}^+(\x_i)$ and $\text{r}(\x_i)$ and the complicated form. Although a few studies have tried to optimize AP for AUPRC optimization~\citep{NIPS2006_af44c4c5,qin2008a,Henderson_2017,Mohapatra2018EfficientOF,Cakir_2019_CVPR,Rolinek_2020_CVPR,DBLP:conf/eccv/BrownXKZ20}, most of them are heuristic driven and do not provide any convergence guarantee.  Recently, \cite{qi2021stochastic} made a breakthrough towards optimizing a differentiable surrogate loss of AP with provable convergence guarantee. They cast the objective as a sum of non-convex compositional functions, and propose a principled stochastic method named SOAP for solving the special optimization problem. A key in their algorithmic design is to maintain and update biased estimators of the surrogate ranking functions 
for all positive data. Theoretical analysis show that  the iteration complexity of SOAP is on the order of $O(1/\epsilon^5)$.

However, it is still unclear whether faster rates than $O(1/\epsilon^5)$ can be obtained. Moreover, whether more advanced update rules such as momentum and  Adam~\citep{kingma2014adam} are useful to accelerate the convergence also remains an open question.     
This paper aims to address these problems, and makes the following contributions.   
\begin{itemize}
    \item We propose momentum-based methods to accelerate the convergence for solving the finite-sum  stochastic compositional optimization problem of AP maximization. The key idea is to employ a momentum update to compute a stochastic gradient estimator of the objective function. 
    
    \item We investigate two schemes for updating the  biased estimators  of the ranking functions, and establish a faster rate in the order of $O(1/\epsilon^4)$ for the iteration complexity. The first is similar to the one proposed by \cite{qi2021stochastic}. However, an improved rate of this scheme is difficult to establish unless the sampled positive data include all positive examples due to a subtle randomness issue. To address this limitation, we propose the second scheme  by updating them in a randomized coordinate-wise fashion. 
  
    \item We propose and analyze a family of adaptive algorithms by using different adaptive step sizes including the Adam-style step size. We establish the same order of iteration complexity by employing the second scheme mentioned above for updating the biased estimators of the ranking functions. 
    {\bf To the best our knowledge, this is the first time} the convergence of Adam-style methods for stochastic compositional problems is established in the literature.  A comparison between our convergence results and  the existing results for maximizing AP is presented in Table~\ref{tab:0}. 
    \item We conduct extensive experiments on benchmark datasets comparing with previous stochastic algorithms for AUPRC/AP optimization and verify the effectiveness of the proposed algorithms. 
\end{itemize}

\begin{table*}[t] 
	\caption{Comparison with previous results for maximizing AP or its surrogate loss. ``-'' indicate results not available or applicable.}\label{tab:0} 
	\vspace{0.5cm}
	\centering
	\label{tab:1}
	\scalebox{0.82}{\begin{tabular}{lcccc}
			\toprule
		 Method	&Provable Convergence&Adaptive Step Size&Iteration Complexity\\
		\hline 
		MOAP (this work)&Yes &No&$O(1/\epsilon^4)$\\
		ADAP (this work)&Yes &\makecell{Adam, AMSGrad, \\AdaGrad, Adabound, etc.}&$O(1/\epsilon^4)$\\
		\hline
	    SOAP (SGD-style)~\citep{qi2021stochastic}&Yes&No&$O(1/\epsilon^5)$\\
	    SOAP (AMSGrad-style)~\citep{qi2021stochastic}&Yes&AMSGrad&$O(1/\epsilon^5)$\\
	    SOAP (Adam-style)~\citep{qi2021stochastic}&No&Adam&-\\
		\bottomrule
	\end{tabular}}
\end{table*}

\section{RELATED WORK}

\paragraph{AUPRC/AP Optimization.}
Many studies in machine learning~\citep{yue2007support,10.5555/3104322.3104421,song2016training}, information retrieval~\citep{metzler2005markov,NIPS2006_af44c4c5,10.5555/2984093.2984129,qin2008a}, computer vision~\citep{Henderson_2017,Mohapatra2018EfficientOF,Cakir_2019_CVPR,Chen_2019_CVPR,NEURIPS2020_b2eeb736,Rolinek_2020_CVPR,DBLP:conf/eccv/BrownXKZ20} have investigated the issue of AUPRC/AP maximization.  Traditional machine learning studies are based on classical optimization techniques, such as hill climbing search \citep{metzler2005markov}, cutting-plane method \citep{yue2007support,10.5555/3104322.3104421}, and dynamic programming \citep{song2016training}. However, these methods do not scale well when the training set is large.  Many studies have considered various methods for computing an (approximiate) gradient for the original AP score function, e.g., finite difference estimators~\citep{Henderson_2017}, linear envelope estimators~\citep{Henderson_2017}, soft histogram binning technique~\citep{Cakir_2019_CVPR},  a blackbox differentiation of a combinatorial solver~\citep{Rolinek_2020_CVPR}, loss-augmented inference for estimating the semi-gradient~\citep{Mohapatra2018EfficientOF}, using the gradient of implicit cost functions~\citep{NIPS2006_af44c4c5}, using the gradient of a smooth approximation of AP~\citep{qin2008a,DBLP:conf/eccv/BrownXKZ20}, etc.  However, none of these studies provide any convergence guarantee when using these techniques for stochastic optimization of AUPRC/AP. 

Recently, \cite{eban2017scalable} propose a systematic framework for AUPRC optimization, which makes use of a finite set of samples to approximate the integral. They then cast the optimization as a min-max saddle-point problem, and optimize it by SGD-style methods without providing any convergence analysis. To the best of our knowledge, \citep{qi2021stochastic} is the first work that proposes a principled stochastic method for maximizing a smooth approximation of the AP function with provable convergence guarantee. 


\paragraph{AUROC Optimization.} Another territory related to AUPRC/AP maximiation is AUROC (aka areas under the ROC curves) maximization. Compared to AUPRC, AUROC is easier to optimize, and the problem of AUROC optimization have been investigated by a large number of previous studies \citep{herschtal2004optimising,joachims2005support,zhao2011online,gao2013one,liu2018fast,natole2018stochastic,ying2016stochastic,yuan2020robust,guo2020communication}. However, an algorithm that maximizes AUROC does not necessarily maximize AUPRC~\citep{davis06}. Hence, we will not directly compare with these studies in the present work.

\paragraph{Stochastic Compositional Optimization.}
The optimization problem  considered in this paper for AP maximization is closely related to stochastic compositional optimization (SCO), where the objective function is of the form $\E_{\xi_1}[f(\E_{\xi_2}[g(\w;\xi_2)];\xi_1)]$, and $\xi_1$ and $\xi_2$ are random variables. SCO has been extensively studied by previous  work \citep{wang2016accelerating,wang2017stochastic,lian2017finite,huo2018accelerated,lin2018improved,liu2018dualityfree,zhang2019stochastic,yang2019multilevel,balasubramanian2020stochastic,DBLP:journals/siamjo/GhadimiRW20,chen2021solving}. The major difference  is that the inner function in our objective depends on both the random variable of the inner level and that of the outer level, which makes the stochastic algorithm design and theoretical analysis much more involved. We notice that a recent work~\citep{hu2020biased} have considered a family of  stochastic compositional optimization problems, where the inner function depends on both the random variable of the inner level and that of the outer level. They proposed biased stochastic methods based on mini-batch averaging. However, their algorithms require a large mini-batch size in the order of $O(1/\epsilon^2)$ for finding an $\epsilon$-stationary solution, which is not practical,  and have a worse sample complexity in the order of $O(1/\epsilon^6)$ or $O(1/\epsilon^5)$. In contrast, our algorithms do not require a large mini-batch size and enjoy  a better sample complexity in the order of $O(m/\epsilon^4)$, where $m$ is the number of positive data points that is usually a moderate number for imbalanced data.

\section{MAIN RESULTS}
\textbf{Notation.} 
In this paper, we consider binary classification problems. Let $(\x,y)$ be a feature-label pair where $\x\in\R^{d}, \y\in\{1, -1\}$, $s_{\w}(\x): \R^d\mapsto \R$ be the score function of a classifier characterized by a parameter vector $\w\in\R^d$. Let $\D_+=\{(\x_1,y_1),\dots,(\x_{m},y_{m})\}$ denote a set of positive examples, and let $\D_-=\{(\x_{m+1}, y_{m+1}),\dots, (\x_{n},y_{n})\}$ denote a set of negative examples. The whole training set is denoted by  $\D=\D_+\cup\D_-$. We denote by $\|\cdot\|$ the Euclidean norm  of a vector. 
Let $\Pi_{\Omega}[\u]= \arg\min_{\bv\in\Omega}\|\bv - \u\|^2$ be the Euclidean projection onto a convex set $\Omega$.

\paragraph{Preliminaries.}
To tackle the non-differentiable ranking functions in~(\ref{eqn:AP}), a differentiable surrogate loss function can be used for approximating the ranking functions. Following~\cite{qi2021stochastic}, we consider the following approximation: 

\begin{align}\label{eqn:AAP}
\widehat{\text{AP}}=\frac{1}{m}\sum_{\x_i\in\D_+}\frac{\sum_{j=1}^n\ell(\w; \x_j, \x_i)\mathbb I(y_j=1) }{\sum_{j=1}^n\ell(\w; \x_j, \x_i)},
\end{align}
where $\ell(\w; \x_j, \x_i)$ is an surrogate function of the indicator function $\mathbb I(s_\w(\x_j)\geq s_\w(\x_i))$. In the literature, various surrogate functions $\ell(\w; \x_j, \x_i)$ have been considered~\citep{DBLP:conf/eccv/BrownXKZ20,qi2021stochastic,qin2008a}, including a squared hinge loss $\ell(\w; \x_j, \x_i )= \max(0, (s_\w(\x_j) - s_\w(\x_i)+ \gamma))^2$, and a logistic function $\ell(\w; \x_j, \x_i) = \frac{\exp(\gamma(s_\w(\x_j) - s_\w(\x_i)))}{1+\exp(\gamma(s_\w(\x_j) - s_\w(\x_i)))}$, where $\gamma$ is a margin or scaling parameter. 
Define the following notations: 
\begin{equation}
	g(\w;\x_j, \x_i)=
	\begin{gathered}
		\begin{bmatrix}
		g_1(\w;\x_j, \x_i)\\
		g_2(\w;\x_j, \x_i)		\end{bmatrix}
		\end{gathered}
	\begin{gathered}
		=\begin{bmatrix}
		\mathbb{I}(y_j=1)\ell(\w;\x_j, \x_i) \\
		\ell(\w;\x_j, \x_i)	
		\end{bmatrix}
	\end{gathered}
,\end{equation} 
and $g_i(\w)=\sum_{j=1}^{n}g(\w;\x_j, \x_i)\in\R^2$, and $f(\mathbf{u})=-\frac{{u}_1}{u_2}:\R^2\mapsto \R$ for any $\mathbf{u}=[u_1,u_2]^{\top}\in\R^2.$ We can see that the two coordinates of $g_i(\w)$ are the surrogates of the two ranking functions $r^+(\x_i)$ and $r(\x_i)$, respectively.  

Then, the optimization problem for AP maximization based on the surrogate function in~(\ref{eqn:AAP}) is equivalent to the following minimization problem:
\begin{equation}
\label{eqn:optimization:problem}
\min\limits_{\w}	 F(\w)=\E_{\x_i\sim\D_+}[f(g_i(\w))]= \frac{1}{m} \sum_{\x_i\in\D_+} f(g_i(\w)). 
\end{equation}
We refer  to the above problem as \emph{finite-sum two-level stochastic compositional optimization}.  We emphasize that a standard $\ell_2$ norm regularizer $\lambda \|\w\|^2$ can be added to the above objective to control over-fitting.  Note that the above problem is generally non-convex even if $g_i(\w)$ is convex. Hence, our goal for solving~(\ref{eqn:optimization:problem}) is to find a nearly stationary solution.  Throughout this paper, we denote by $ \tilde{g}_{i}(\w)$ and $ \nabla\tilde{g}_{i}(\w)$  independent unbiased stochastic estimators of $g_i(\w)$ and $\nabla g_i(\w)$, respectively, which are computed based on sampled (mini-batch) data from all data points in $\D$. For example, we can sample two sets of examples from $\D$ denoted by $\S_1$ and $\S_2$ and compute $ \tilde{g}_{i}(\w) = \frac{n}{|\S_1|}\sum_{\x_j\in\S_1}g(\w; \x_j, \x_i)$ and $ \nabla\tilde{g}_{i}(\w) = \frac{n}{|\S_2|}\sum_{\x_j\in\S_1}\nabla g(\w; \x_j, \x_i).$
To this end, we impose the following  assumptions~\citep{qi2021stochastic}. 
\begin{ass}
\label{ass:bounded:ell}
{Assume (i) there exist $C>0$ such that $\ell(\w;\x_i;\x_i)\geq C$, for any $\x_i\in\D_+$; (ii) $\ell(\w;\x_j;\x_i)$ is $C_{\ell}$-Lipschitz continuous and $L_{\ell}$-smooth with respect to $\w$ for any $\x_i\in\D_+$, $\x_j\in\D$, where $C_{\ell}, L_{\ell}>0$ are constants; (iii) it holds that  $0\leq \ell(\w;\x_j;\x_i)\leq M$ for some $M>0$ for any $\x_i\in\D_+$, $\x_j\in\D$. }
\end{ass}
\begin{ass}
\label{ass:bounded:g}
Assume there exists a positive constant $\sigma$, such that  $\|\nabla g(\w; \x_j, \x_i)\|^2\leq  \sigma^2$ for any  $
\x_i\in\D_+, \x_j\in \D$. 
\end{ass}
\begin{Remark} The above assumptions can be easily satisfied for a smooth surrogate loss function $\ell(\cdot; \x_j, \x_i)$ and a bounded score function $s_\w(\x)$. For example, for a linear model we can use $s_\w(\x)= 1/(1+\exp(-\w^{\top}\x))\in[0,1]$ as the score function. 
\end{Remark}

Based on Assumption \ref{ass:bounded:ell}, we can establish the smoothness of the objective function in the optimization problem \eqref{eqn:optimization:problem}, and show that 
$g_i(\w)\in\Omega=\{\u\in\R^2, u_1\leq Mm, C\leq u_2\leq Mn\}$ (cf. the supplement). 
Our goal is then to find a solution $\w$ such that $\|\nabla F(\w)\|\leq \epsilon$ in expectation, to which we refer as $\epsilon$-stationary solution.  


Before ending this section, we summarize the updates of SOAP algorithm~\citep{qi2021stochastic}  to facilitate comparisons. An essential component of SOAP is to compute an estimator of the gradient of the objective function, i.e.,   $\nabla F(\w)=\frac{1}{m}\sum_{i=1}^m \nabla g_i(\w_t)^{\top}\nabla f(g_i(\w_t))$ at each iteration. Since both $g_i(\w_t)$ and $\nabla g_i(\w_t) $ involves processing all examples,  stochastic estimators of these quantities are computed. In addition, due to the fact that $g_i(\w_t)$ is inside $\nabla f$  a better technique than a simple mini-batch averaging is used in order to control the variance. To this end, they introduce a sequence $U_t\in\R^{m\times 2}$ for tracking $g(\w) = [g_1(\w), \ldots, g_m(\w)]^{\top}\in\R^{m\times 2}$, and proposes the SOAP algorithm with the following  updates: 
\begin{equation}
\text{SOAP}\left\{
\begin{aligned}
&\text{Sample $B$ points from $\D_+$, denoted by $\B_t$}\\
&[U_{t+1}]_i =\left\{\begin{array}{ll} (1-\beta)[U_t]_i + \beta \tilde{g}_{i}(\w_t)^{\top} & i\in\B_t\\\
[U_t]_i & \text{o.w.}  \end{array}\right.\\
&\tilde{\nabla}_t=\frac{1}{B}\sum_{i\in{\B_t}}  \nabla\tilde{g}_{i}(\w_t)^{\top}\cdot \nabla f([U_{t+1}]^{\top}_i)\\
&\w_{t+1}=\w_t-\eta \tilde{\nabla}_t
\end{aligned}\right.
\end{equation}
where $[U_t]_i$ denotes the $i$-th row of $U_t$.   The  $U_{t+1}$ sequence is known as the moving average estimator in the literature of stochastic compositional optimization~\citep{wang2017stochastic}.   SOAP enjoys an iteration complexity of $O(1/\epsilon^5)$ for using the above update to find an $\epsilon$-stationary solution of the objective when $\B_t = \D_+$ and an iteration complexity of $O(m/\epsilon^5)$ when $|\B_t| =O(1)$. Below, we present novel algorithms to improve these complexities.

 \subsection{Stochastic Momentum Methods for AP Maximization: MOAP} 
 \begin{algorithm}[t]
   \caption{MOAP-V1}
   \label{alg:AP-MSGD-V1}
\begin{algorithmic}[1]
\STATE {\bf Input:} $\eta$, $\beta_0, \beta_1$, $B$
\STATE {\bf Initialize:} $\w_1\in\R^d, U_1\in\R^{m\times2}, \m_1\in\R^d$ 
\FOR {$t=1,\dots,T$}
\STATE Sample $B$ points from $\D_+$, denoted by $\B_t$
\STATE Set $[U_{t+1}]_i$ as:\\ 
$\left\{\begin{array}{ll} \Pi_{\Omega}[(1-\beta_0)[U_t]_i + \beta_0 \tilde{g}_{i}(\w_t)^{\top}] & i\in\B_t\\\
[U_t]_i & \text{o.w.}  \end{array}\right.$
\STATE $\tilde{\nabla}_t=\frac{1}{B}\sum_{i\in{\B_t}}  \nabla \tilde{g}_{i}(\w_t)^{\top}\cdot \nabla f([U_{t+1}]^{\top}_i)$
\STATE $\m_{t+1}=(1-\beta_1)\m_t+\beta_1\tilde{\nabla}_t  $
\STATE $\w_{t+1}=\w_t-\eta \m_{t+1}$
\ENDFOR
\end{algorithmic}
\end{algorithm}

To improve the convergence of SOAP, we propose to exploit momentum when updating the model parameter $\w_t$.  Similar to the stochastic momentum method widely used in stochastic optimization~\citep{yangnonconvexmo}, we maintain and update the following stochastic gradient estimator: 
\begin{align}\label{eqn:mom}
\m_{t+1}=(1-\beta_1)\m_t+\beta_1\tilde{\nabla}_t,
\end{align}
where $\tilde{\nabla}_t$ is a stochastic estimator of the gradient at $\w_t$, which is computed  based on $U_{t+1}$ similarly to SOAP. Then, we update the solution by $\w_{t+1} = \w_t - \eta \m_{t+1}$. We investigate two methods for updating $U_{t+1}$.

\paragraph{MOAP-V1.} The first method for updating $U_{t+1}$ is similar to that in \citep{qi2021stochastic}, which is presented in Algorithm \ref{alg:AP-MSGD-V1}. It is worth mentioning that we conduct a projection $\Pi_{\Omega}[\cdot]$ in Step 5 to ensure the two components of each row in $U$ is appropriately bounded such that $\nabla f(\cdot)$ is Lipschitz continuous with respect to its input.   
Regarding the convergence of Algorithm \ref{alg:AP-MSGD-V1}, we first present the following result by setting $\B_t = \D_+$ at every iteration. 
\begin{tho}
\label{thm:msam:v1}
Suppose Assumptions \ref{ass:bounded:ell} and \ref{ass:bounded:g} hold. Then, Algorithm \ref{alg:AP-MSGD-V1} with $\B_t=\D_+$, $\beta_0= O(\epsilon^2), \beta_1=O(\epsilon^2)$, $\eta=O(\epsilon^2)$, and $T=O(1/\epsilon^4)$ ensures
$$\E\left[\frac{1}{T}\sum_{t=1}^{T} \| \nabla F(\w_t)\|^2\right]\leq \epsilon^2.$$	
\end{tho}
\begin{Remark}
\emph{Compare with Theorem 2 in~\cite{qi2021stochastic}, the MOAP has a better iteration complexity, i.e., $T=O(1/\epsilon^4)$ vs $T=O(1/\epsilon^5)$ of SOAP. The total sample complexity of MOAP in this case is $O(m/\epsilon^4)$. We include all omitted proofs in the supplementary materials and exhibit the constants in the big $O$ in the proof. }
\end{Remark}


\begin{algorithm}[t]
   \caption{MOAP-V2}
   \label{alg:AP-MSGD-V2}
\begin{algorithmic}[1]
\STATE {\bf Input:} $\eta$, $\beta_0, \beta_1$, $B$
\STATE {\bf Initialize:} $\w_1\in\R^d, U_1\in\R^{m\times2}, \m_1\in\R^d$ 
\FOR {$t=1,\dots,T$}
\STATE Sample $B$ points from $\D_+$, denoted by $\B_t$
\STATE Set $[U_{t+1}]_i$ as:\\
$\left\{\begin{array}{ll}\Pi_{\Omega}[ (1-\beta_0)[U_t]_i + \beta_0\frac{m}{B} \tilde{g}_{i}(\w_t)^{\top}] & i\in\B_t\\
\Pi_{\Omega}[(1-\beta_0)[U_t]_i] & \text{o.w.}  \end{array}\right.$
\STATE $\tilde{\nabla}_t=\frac{1}{B}\sum_{i\in{\B_t}}  \nabla\tilde{g}_{i}(\w_t)^{\top}\cdot \nabla f([U_{t+1}]^{\top}_i)$
\STATE $\m_{t+1}=(1-\beta_1)\m_t+\beta_1 \tilde{\nabla}_t $
\STATE $\w_{t+1}=\w_t-\eta \m_{t+1}$
\ENDFOR
\end{algorithmic}
\end{algorithm}
Next, we consider using random samples $\B_t\sim\D$. Without loss of generality, we assume $|\B_t|=1$ and the sample is randomly chosen from $\D$ with replacement. Then, for Algorithm 1, we provide the following convergence rate. 
\begin{tho}
\label{thm:msam:v2}
Suppose Assumptions  \ref{ass:bounded:ell} and \ref{ass:bounded:g} hold. Then, Algorithm \ref{alg:AP-MSGD-V1} with $B=1$, $\beta_0 = O(\epsilon^2), \beta_1=O(\epsilon^2)$, $\eta=O(\epsilon^3/m)$, $T=O(m/\epsilon^5)$
satisfies
$$\E\left[\frac{1}{T}\sum_{t=1}^{T} \| \nabla F(\w_t)\|^2\right]\leq \epsilon^2.$$	
\end{tho}
\begin{Remark} \emph{
Theorem \ref{thm:msam:v2} implies that, for a stochastic $\B_t$, Algorithm \ref{alg:AP-MSGD-V1} suffers a worst case iteration complexity in the order of $O(m/\epsilon^5)$, which is the same as SOAP for using a stochastic $\B_t$. This is mainly due to a subtle dependent issue caused by the \emph{infrequent updates} of the moving average estimator $U_t$, which makes the momentum fail to accelerate the convergence. The detailed discussion is presented in the supplementary.}
\end{Remark} 


\paragraph{MOAP-V2.} To address the limitation of Algorithm \ref{alg:AP-MSGD-V1} for using a  stochastic $\B_t$, we propose a second method for updating $U_{t+1}$. The procedure is presented in Algorithm \ref{alg:AP-MSGD-V2}. Different from MOAP-V1, MOAP-V2 updates all coordinates of $U_{t+1}$ according to 
\begin{align*}
[U_{t+1}]_i =\left\{\begin{array}{ll} \Pi_{\Omega}[(1-\beta_0)[U_t]_i + \beta_0\frac{m}{B} \tilde{g}_{i}(\w_t)^{\top}] & i\in\B_t\\
\Pi_{\Omega}[(1-\beta_0)[U_t]_i ]& \text{o.w.}  \end{array}\right.,
\end{align*}
i.e, the coordinate of $U_{t+1}$ corresponding to a sampled positive data $i\in\B_t$ is updated similarly to that in MOAP-V1 except the unbiased estimator $ \tilde{g}_{i}(\w_t)$ is scaled by $m/B$, and the coordinate of $U_{t+1}$ corresponding to a non-sampled data $i\in\B_t$ is updated trivially by multiplying with a scaling factor $1 - \beta_0$. It is notable that we can delay updating these coordinates until they are sampled because at the current iteration, these coordinates are not used for updating the model parameter $\w_{t+1}$. 


In order to understand why this method can help improve the convergence for using a stochastic $\B_t$. We can write the update of $U_{t+1}$ equivalently as  $U_{t+1}=\Pi_{\Omega^m}[(1-\beta_0)U_t + \beta_0 \widehat{g}(\w_t)]$, 
where $\widehat g(\w_t)$ is defined as
\begin{align}\label{eqn:cord}
    \widehat g(\w_t)= \frac{1}{B}\sum_{i\in\B_T}\left(\begin{array}{c}0\\\cdot\\ m\tilde g_i(\w_t)^{\top}\\\cdot\\ 0\end{array}\right).
\end{align}
It is not difficult to show that $\E[\widehat g(\w_t)] = g(\w_t)$, where the expectation is taken over the randomness in $\B_t$ and $\tilde g_i(\w_t)$. We refer to this update of $U_{t+1}$ as {\bf the randomized coordinate update}. This update is similar to Step 5  in MOAP-V1 when using $\B_t= \D_+$ to update the model, in which case all coordinates of $U_{t+1}$ are updated by $U_{t+1}= \Pi_{\Omega^m}[(1-\beta_0)U_t + \beta_0 \widetilde{g}(\w_t)]$, 
where $\widetilde g(\w_t) = (\tilde g_1(\w_t), \ldots, \tilde g_m(\w_t))^{\top}$. Both $ \widehat g(\w_t)$ and  $\widetilde g(\w_t)$ are unbiased estimators of $g(\w_t)$. The difference is that the variance of $\widehat g(\w_t)$ is scaled up by a factor of $\frac{m}{B}$. We emphasize it is the combination of the momentum update~(\ref{eqn:mom}) and the randomized coordinate update~(\ref{eqn:cord}) that yields an improved convergence rate of Algorithm \ref{alg:AP-MSGD-V2} presented below.  
\begin{tho}\label{thm:3}
Suppose Assumptions \ref{ass:bounded:ell} and \ref{ass:bounded:g} hold. Then, Algorithm 2 with  $\beta_0=O(\epsilon^2B/m)$, $\beta_1=O(\epsilon^2B/m)$, $\eta=O(\beta)$, $T=O(m/[B\epsilon^4])$ ensures that
$\E\left[\frac{1}{T}\sum_{t=1}^{T} \| \nabla F(\w_t)\|^2\right]\leq \epsilon^2$.	
\end{tho}
\begin{Remark} \emph{
The above theorem implies that, for a stochastic $\B_t$, Algorithm \ref{alg:AP-MSGD-V2} enjoys  a better iteration complexity in the order of $O(m/\epsilon^4)$, which is the better than SOAP for using a stochastic $\B_t$. 
}
\end{Remark} 
 
\subsection{Stochastic Adaptive  Algorithms for AP Maximization: ADAP}
\begin{algorithm}[t]
   \caption{ADAP}
   \label{alg:Adam}
\begin{algorithmic}[1]
\STATE {\bf Input:} $\eta$, $\beta_0, \beta_1, \delta$, $B$
\STATE {\bf Initialize:} $\w_1\in\R^d, U_1\in\R^{m\times2}, \m_1\in\R^d$ 
\FOR {$t=1,\dots,T$}
\STATE Sample $B$ points from $\D_+$, denoted by $\B_t$
\STATE Set $[U_{t+1}]_i$ as:\\  $\left\{\begin{array}{ll} \Pi_{\Omega}[(1-\beta_0)[U_t]_i + \beta_0\frac{m}{B} \tilde{g}_{i}(\w_t)^{\top}] & i\in\B_t\\
\Pi_{\Omega}[(1-\beta_0)[U_t]_i] & \text{o.w.}  \end{array}\right.$
\STATE $\tilde{\nabla}_t=\frac{1}{B}\sum_{i\in \B_t}  \nabla_{\w}\tilde{g}_{i,t}(\w_t)^{\top}\cdot \nabla f([U_{t+1}]^{\top}_i)$
\STATE $\m_{t+1}=(1-\beta)\m_t+\beta \tilde{\nabla}_t $
\STATE $\bv_{t+1}=h_t(\bv_t,\tilde{\nabla}_t^2)$ \hfill $\diamond h_t$ can be implemented by any of the methods in~(\ref{eqn:ada})
\STATE $\w_{t+1}=\w_t-\frac{\eta}{\sqrt{\bv_{t+1}}+\delta} \m_{t+1}$
\ENDFOR
\end{algorithmic}
\end{algorithm}

In this section, we extend our technique to stochastic adaptive algorithms, which use  adaptive step sizes. For standard stochastic optimization, various adaptive step sizes have been investigated, including AdaGrad~\citep{duchi2011adaptive}, Adam~\citep{kingma2014adam}, AMSGrad~\citep{reddi2019convergence}, AdaBound~\citep{luo2018adaptive}, etc.

Motivated by these existing adaptive methods, we propose stochastic adaptive algorithms for AP maximization in Algorithm \ref{alg:Adam}, which is referred to as ADAP. The difference from Algorithm~\ref{alg:AP-MSGD-V2} is that we need to maintain and update another sequence $\bv_t \in\mathbb R^d$ in order to compute an adaptive step size in Step 8, where $h_t(\cdot)$ denotes a generic updating function. The $\bv_t$ is usually computed from the second-order moment (i.e., coordinate-wise square of the stochastic gradient estimator). For our problem, we compute $\bv_t$ from $\tilde{\nabla}_t^2$, and $h_t(\cdot)$ can be implemented by the following methods for having different adaptive step sizes: 
\begin{equation}\label{eqn:ada}
\begin{aligned}
\text{Adam-style:} \: \bv_{t+1} &= (1-\beta_t')\bv_t + \beta_t'\tilde{\nabla}_t^2\\
\text{AdaGrad-style:} \: \bv_{t+1} &= \frac{1}{t+1}\sum_{i=1}^t\tilde{\nabla}_i^2\\
\text{AMSGrad-style:} \: \bv'_{t+1} &= (1-\beta_t')\bv'_t + \beta_t' \tilde{\nabla}_t^2,\\ \bv_{t+1} &= \max(\bv_t, \bv'_{t+1})\\
\text{AdaBound-style:} \: \bv'_{t+1} &= (1-\beta_t')\bv'_t + \beta_t'\tilde{\nabla}_t^2,\\ \bv_{t+1} &= \Pi_{[1/c_u^2, 1/c_l^2]}[\bv'_{t+1}]
\end{aligned}
\end{equation}
where $c_l<c_u$ are two parameters of the AdaBound method and $\Pi_{[a,b]}$ denotes a clipping operator that clips the input argument into the range $[a, b]$. 
Given $\bv_{t+1}$, we update the model parameter by 
$$\w_{t+1}=\w_t-\frac{\eta}{\sqrt{\bv_{t+1}}+\delta} \m_{t+1},$$
where $\delta>0$ is a parameter. 
The convergence analysis of ADAP relies on the boundness of the  step size scaling factor $\eta^s_t = 1/(\sqrt{\bv_{t+1}} +\delta)$ following similarly a recent study on the convergence of the Adam-style algorithms for standard stochastic non-convex optimization~\citep{DBLP:journals/corr/abs-2104-14840}. Specifically, under Assumptions \ref{ass:bounded:ell} and \ref{ass:bounded:g}, we could show that $\|\tilde{\nabla}_t\|_\infty$ are bounded by a constant, which further implies that the step size scaling factor $\eta^s_t = 1/(\sqrt{\bv_{t+1}} + \delta)$ is upper bounded and lower bounded by some constants $c_{u}$ and $c_{l}$. 
Finally, we present the convergence of  Algorithm  \ref{alg:Adam} below. 
\begin{tho}\label{thm:4}
Suppose Assumptions \ref{ass:bounded:ell} and \ref{ass:bounded:g}  hold. Then, Algorithm \ref{alg:Adam} with $\beta_0=O(B\epsilon^2/m)$, $\beta_1=O(B\epsilon^2/m)$, $\eta=O(\beta)$, $T=O(m/[B\epsilon^4])$ and any of the methods in~(\ref{eqn:ada}) for computing $\bv_{t+1}$ ensures that
$\E\left[\frac{1}{T}\sum_{t=1}^{T} \| \nabla F(\w_t)\|^2\right]\leq \epsilon^2.$	
\end{tho}
\begin{Remark}

\emph{We note that previous work \citep{qi2021stochastic} also proposed a Adam-style algorithm, but their analysis only covers the AMSGrad-style adaptive step size and has a worse iteration complexity in the order of $O(m/\epsilon^5)$ for using a stochastic $\B_t$.  }
\end{Remark}

Finally, it is worth mentioning that both Theorem~\ref{thm:3} and Theorem~\ref{thm:4} can be extended to using a decreasing step size $\eta$ and parameters $\beta_0, \beta_1$ in the order of $1/\sqrt{t}$, where $t$ denotes the current round.  We refer the readers to the supplement for the proofs.

\section{EXPERIMENTS}
\label{section:exp}
\begin{table*}[t] 
	\caption{Statistics of Datasets.}\label{exp:tab:2} 
	\centering
	\label{exp:tab:1}
	{\begin{tabular}{c|c|c|c}
			\toprule
		 Data Set	&\#training examples&\#testing examples&Proportion of positive data\\
		\hline 
	mushrooms&$2920$ &$1504$&$5.27\%$\\
		phishing&$4987$ &$2568$&$35.65\%$\\
		w6a&17188 &32561&$3.05\%$\\
	    a9a&$32561$ &$16281$&$24.08\%$\\
	    w8a&$49749$ &$14951$&$2.97\%$\\
		ijcnn1&$49990$ &$91701$&$9.71\%$\\
		\bottomrule
	\end{tabular}}
\end{table*}

In this section, we conduct experiments on benchmark datasets to demonstrate the effectiveness of the proposed algorithms. More results can be found in the Appendix.
\subsection{Optimizing Linear Models}
We first consider learning a linear model for prediction. For baselines, we choose three state-of-the-art methods for stochastic optimization of AP,  namely SOAP (SGD)~\citep{qi2021stochastic}, SOAP (Adam)~\citep{qi2021stochastic} and SmoothAP~\citep{DBLP:conf/eccv/BrownXKZ20}. For our algorithms, we implement MOAP (v2) and ADAP (Adam-style). 

\paragraph{Datasets.} We use six imbalanced benchmark datasets from LIBSVM data~\citep{CC01a111}, whose statistics are summarized in the appendix. For all data sets, we scale the feature vectors into $[0,1]$. For \emph{mushrooms} and \emph{phishing} dataset, since the original class distribution is relatively balanced, and no testing set is given, we randomly drop a set of positive data from the training examples, and divide them into  training and testing according to 2:1 ratio. 

\paragraph{Configurations.} In all experiments, we use the sigmoid function  $1/(\exp(-\w^{\top}\x)+1)$ to generate a prediction score for computing the AP. We set the $\ell_2$-regulation parameter as $10^{-4},$  the mini-batch size as  20, and run a total of $T=500$ iterations. For MOAP, ADAP, SOAP, we choose the squared hinge ls as the surrogate function following \cite{qi2021stochastic}. For SmoothAP,  we apply the sigmoid function to approximate the indicator function following \cite{DBLP:conf/eccv/BrownXKZ20}. Other involved parameters for each algorithm are tuned on the training data. For MOAP and ADAP,  we decrease $\eta$ and $\beta_0=\beta_1$ on the order of $O(1/\sqrt{t})$ according to the theoretical analysis, and tune the initial value of the step size in the set $\{20,10,1,10^{-1},10^{-2}\}$, and $\beta_0=\beta_1$ in the set $\{0.9,0.5,0.1\}$. For other algorithms, we observe poor performance when using a decreasing step size and $\beta$ parameters, and thus report their results by using fixed parameters tuned in the same range as our algorithms. We repeat each algorithm 5 times on each data and report the averaged results.  

\begin{figure*}
\centering
\subfigure[mushrooms dataset]{
\label{Fig.2sub.3}
\includegraphics[width=0.3\textwidth]{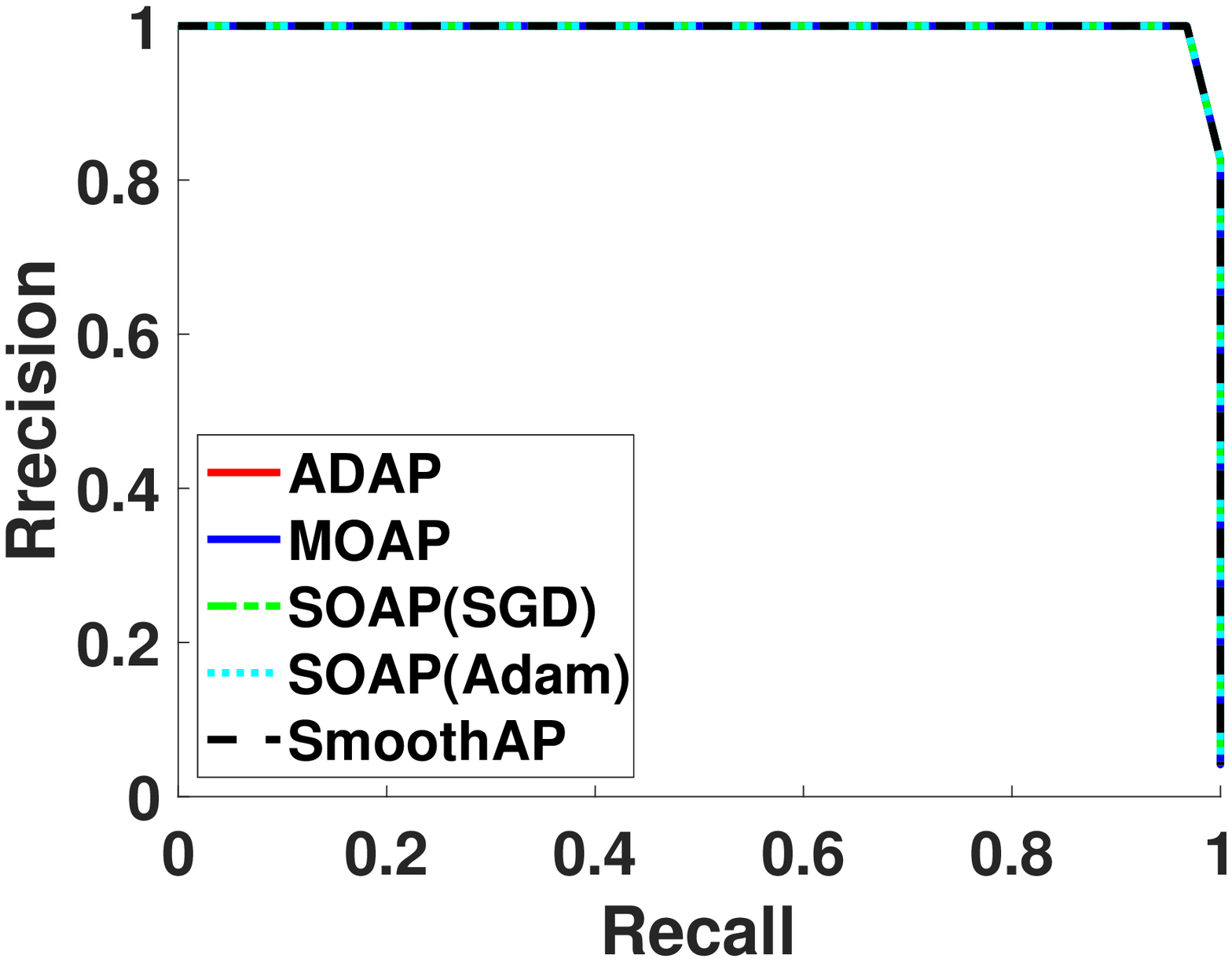}}
\subfigure[phishing dataset]{
\label{Fig.2sub.4}
\includegraphics[width=0.3\textwidth]{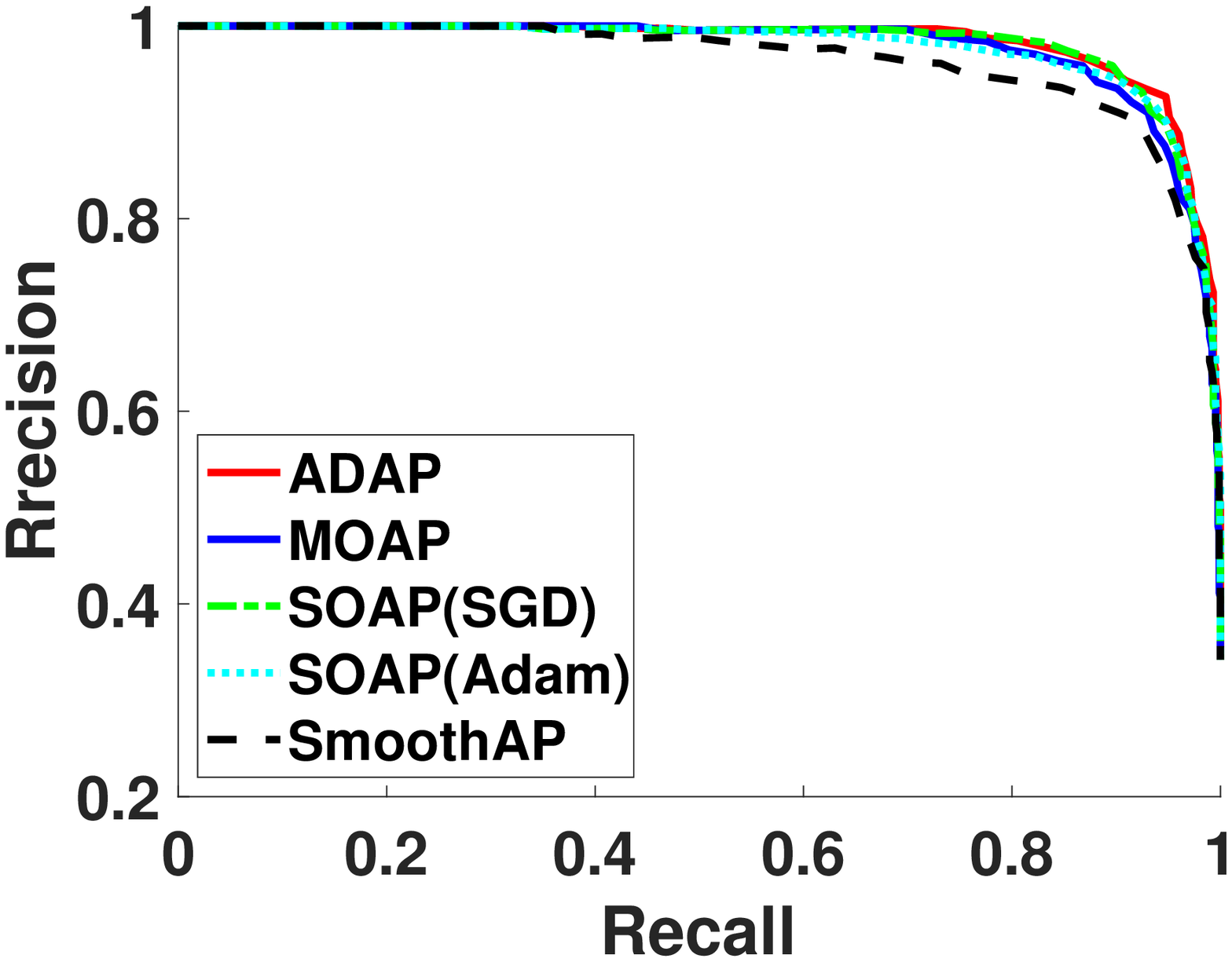}} 
\subfigure[w6a dataset]{
\label{Fig.2sub.5}
\includegraphics[width=0.3\textwidth]{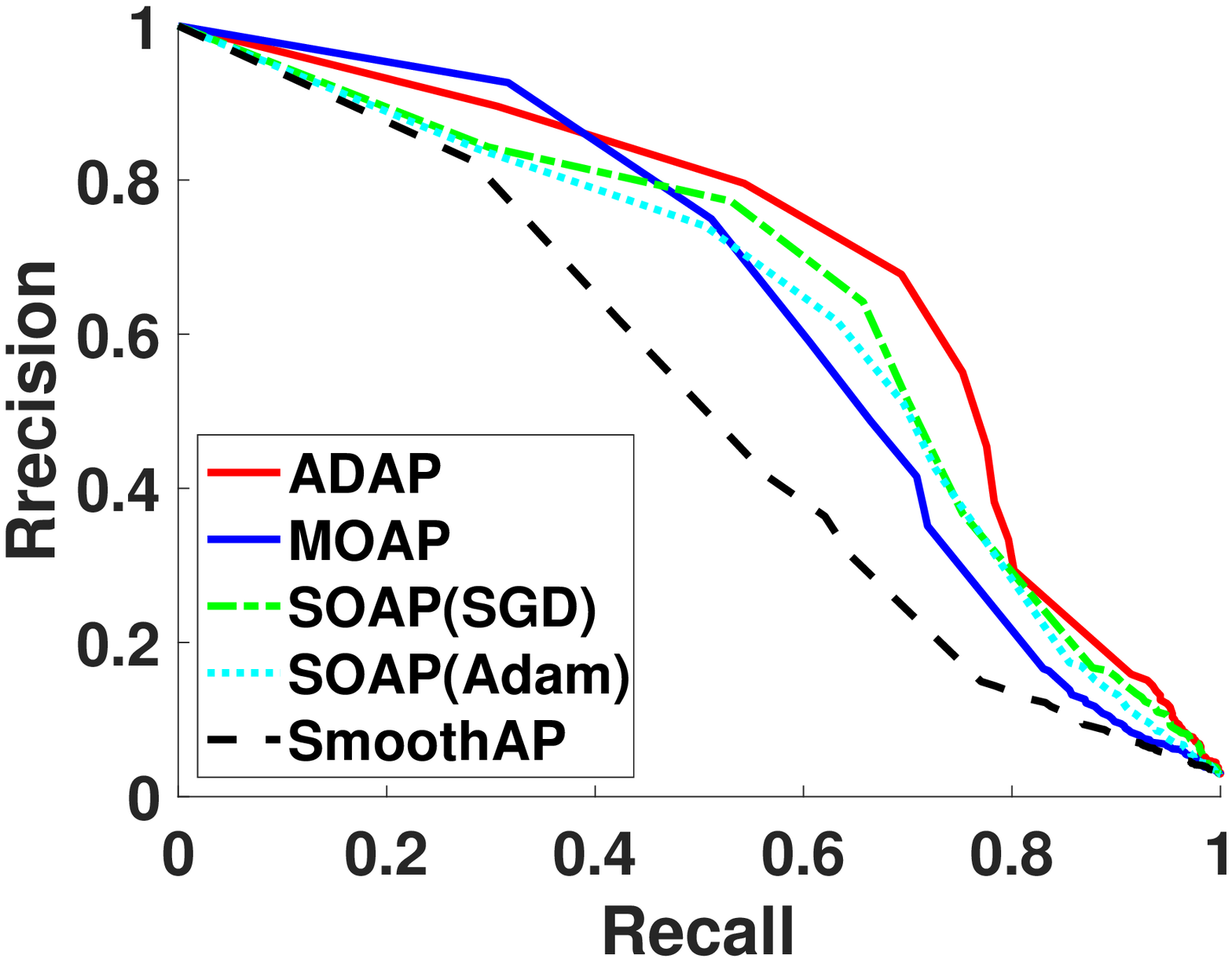}} 
\subfigure[a9a dataset]{
\label{Fig.2sub.2}
\includegraphics[width=0.3\textwidth]{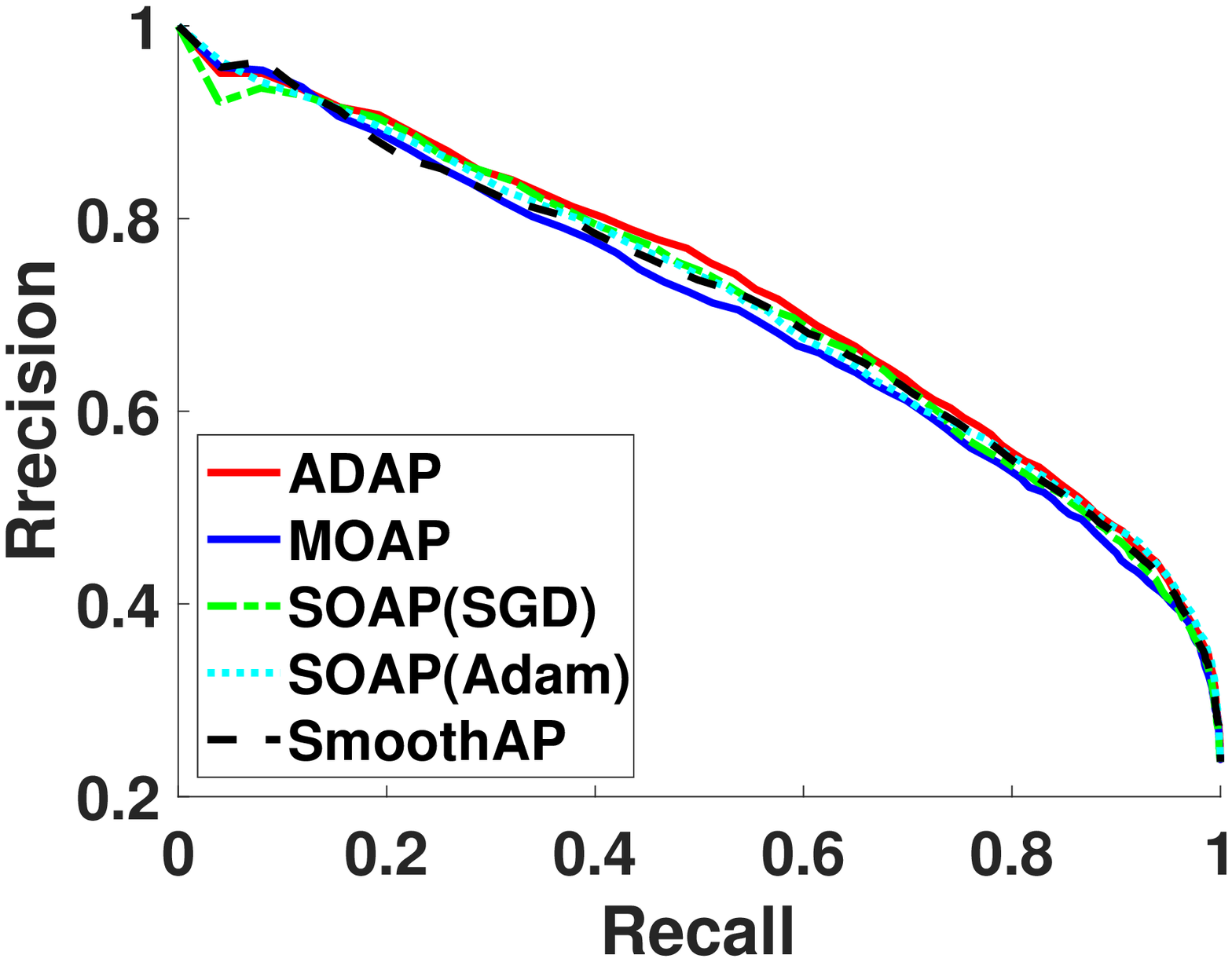}}
\subfigure[w8a dataset]{
\label{Fig.2sub.1}
\includegraphics[width=0.3\textwidth]{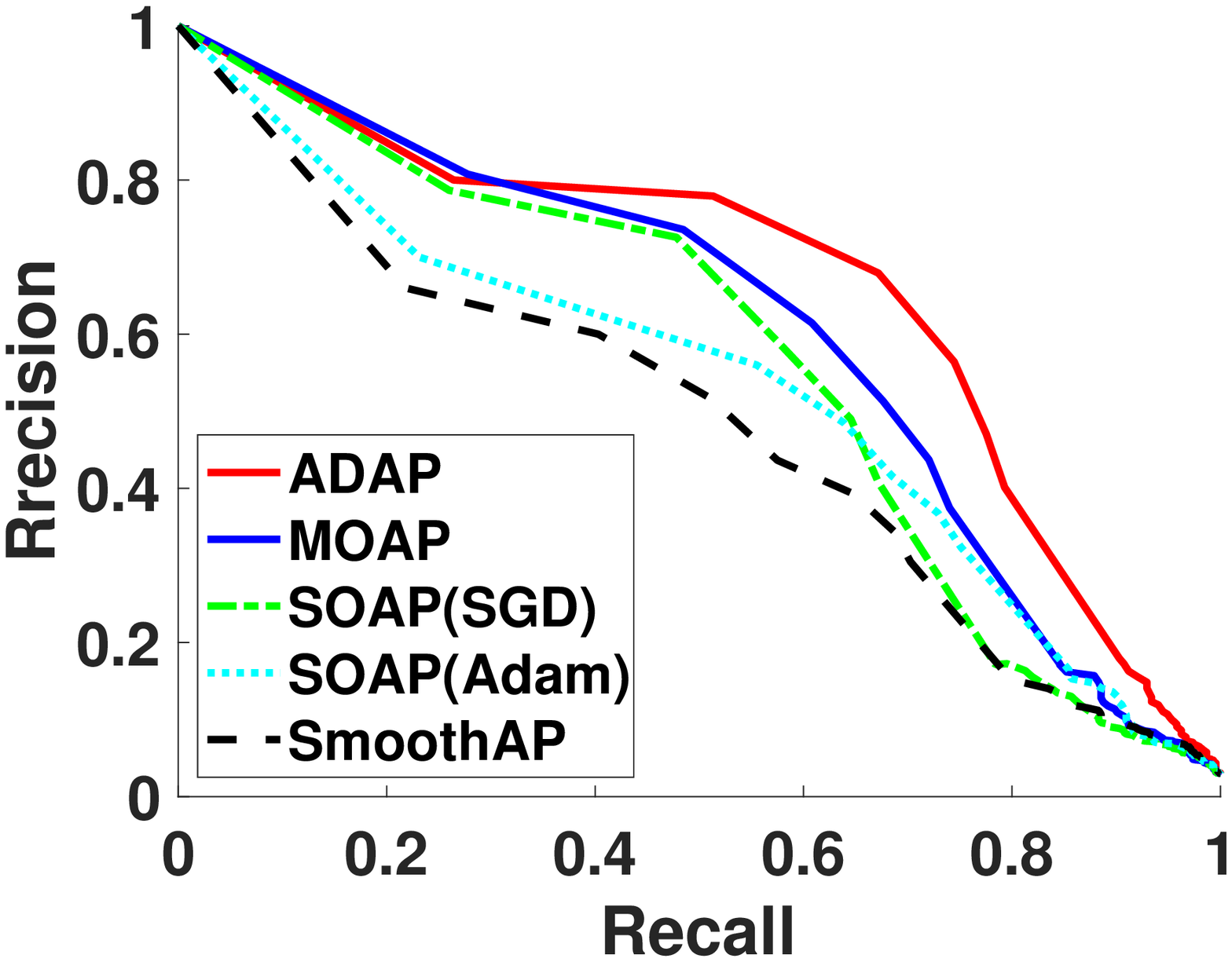}} 
\subfigure[ijcnn1 dataset]{
\label{Fig.2sub.6}
\includegraphics[width=0.3\textwidth]{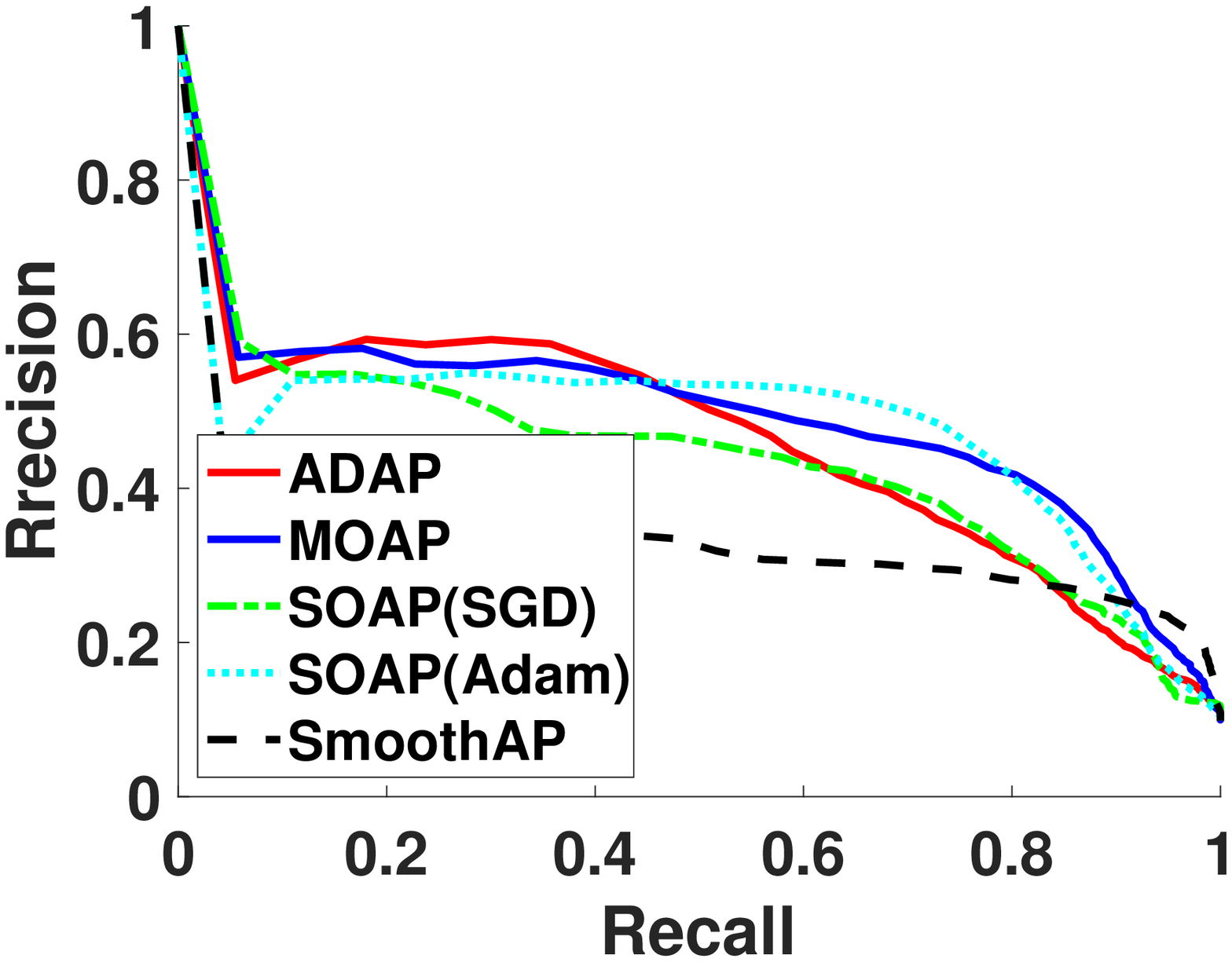}} 
\caption{Precision-Recall curves of the Final models on the testing set}
\label{Fig.main2}
\end{figure*}
\begin{table*}[t] 
	\caption{Final averaged AP scores on the testing data.}\label{exp:tab:3} 
	\centering
	\label{tab:1}
	{\begin{tabular}{l|c|c|c}
			\toprule
		 Method	&mushrooms&phishing&w6a\\
		\hline 
		MOAP&$0.999\pm 2$E-$6$&$0.978\pm2$E-$6$&$0.596\pm2$E-$3$\\
		ADAP&${\bf 1}$&${\bf 0.981\pm2}${\bf E-}${\bf 7}$&${\bf 0.675\pm1}$\bf{E-}${\bf 4}$\\
		\hline
	    SOAP (SGD)&$0.997\pm1$E-$5$&$0.978\pm4$E-$6$&$0.612\pm4$E-$5$\\
	    SOAP (Adam)&$1$&$0.977\pm1$E-$6$&$0.592\pm3$E-$4$\\
	    SmoothAP& $0.962\pm1$E-$3$&$0.967\pm8$E-$6$&$0.473\pm1$E-$4$ \\
		\midrule

		\midrule
		 Method	&a9a&w8a&ijcnn1\\
		\hline 
		MOAP&$0.715\pm3$E-$5$& $0.584\pm 4$E-$4$ &$0.471\pm6$E-$4$ \\
		ADAP&${\bf 0.730\pm1}${\bf E-}${\bf 6}$& ${\bf 0.659\pm5}${\bf E-}${\bf 4}$&$0.536\pm1$E-$3$\\
		\hline
	    SOAP (SGD) &$0.714\pm3$E-$5$&$ 0.561 \pm 1$E-$3$&$0.476\pm2$E-$4$\\
	    SOAP (Adam) &$0.721\pm1$E-$5$&$0.475\pm2$E-$3$&${\bf 0.547\pm1}${\bf E-}${\bf 3}$\\
	    SmoothAP &$0.713\pm 3$E-$5$&$0.459\pm3$E-$4$& $0.374\pm3$E-$4$\\
		\bottomrule
	\end{tabular}
}
\end{table*}
\begin{figure*}[h]
\centering  
\subfigure[mushrooms dataset]{
\label{Fig.sub.3}
\includegraphics[width=0.31\textwidth]{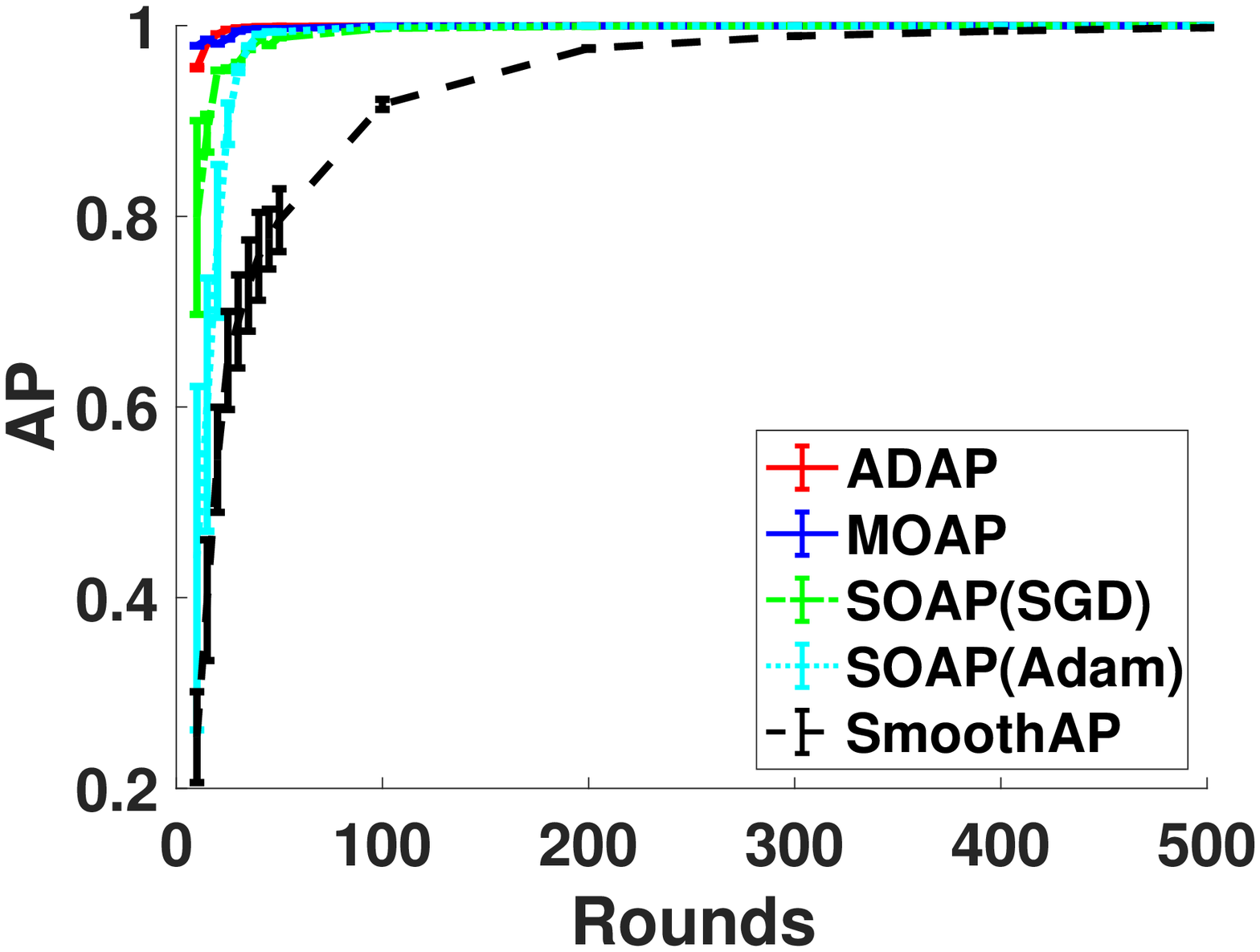}}
\subfigure[phishing dataset]{
\label{Fig.sub.4}
\includegraphics[width=0.31\textwidth]{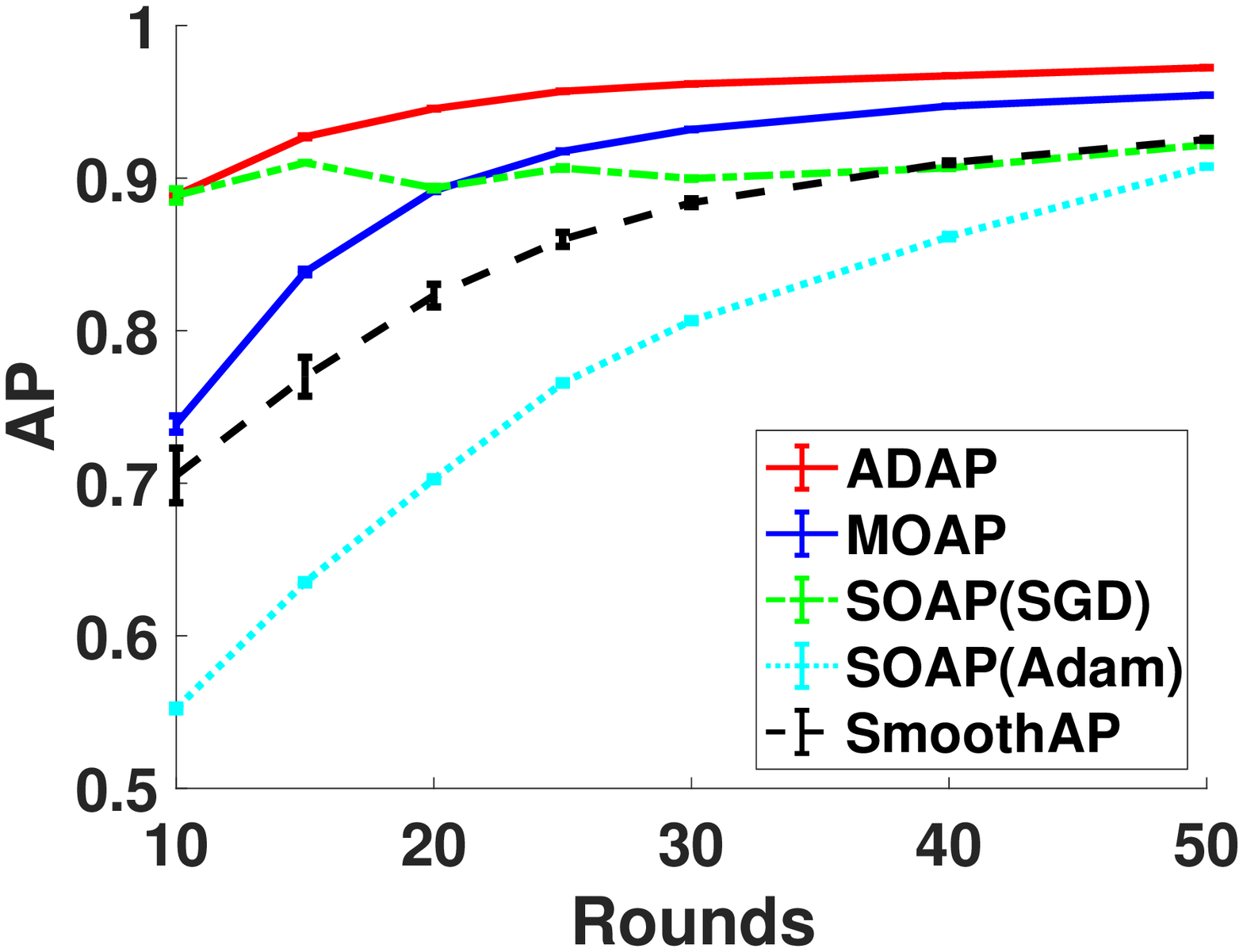}} 
\subfigure[w6a dataset]{
\label{Fig.sub.5}
\includegraphics[width=0.31\textwidth]{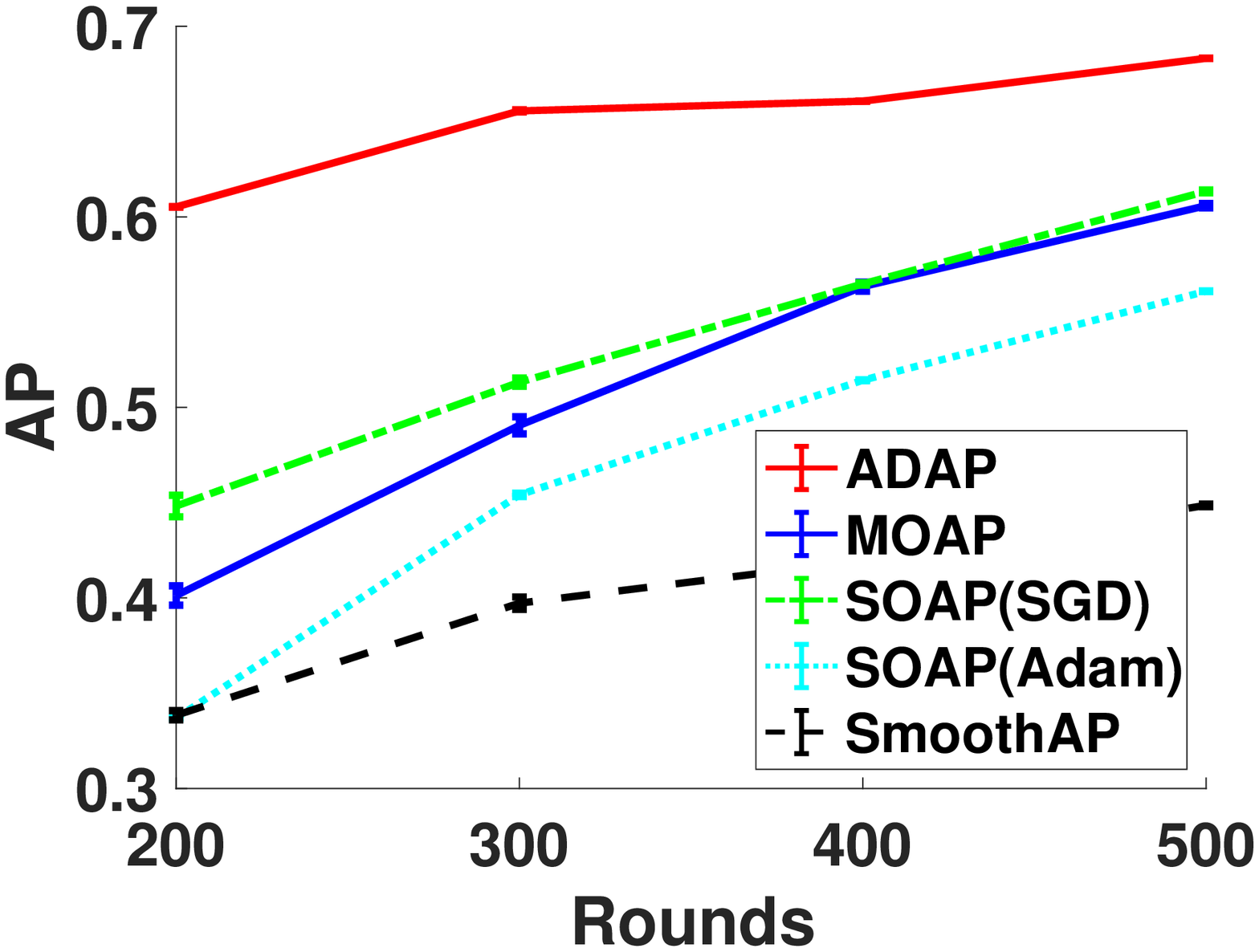}} 
\subfigure[a9a dataset]{
\label{Fig.sub.2}
\includegraphics[width=0.31\textwidth]{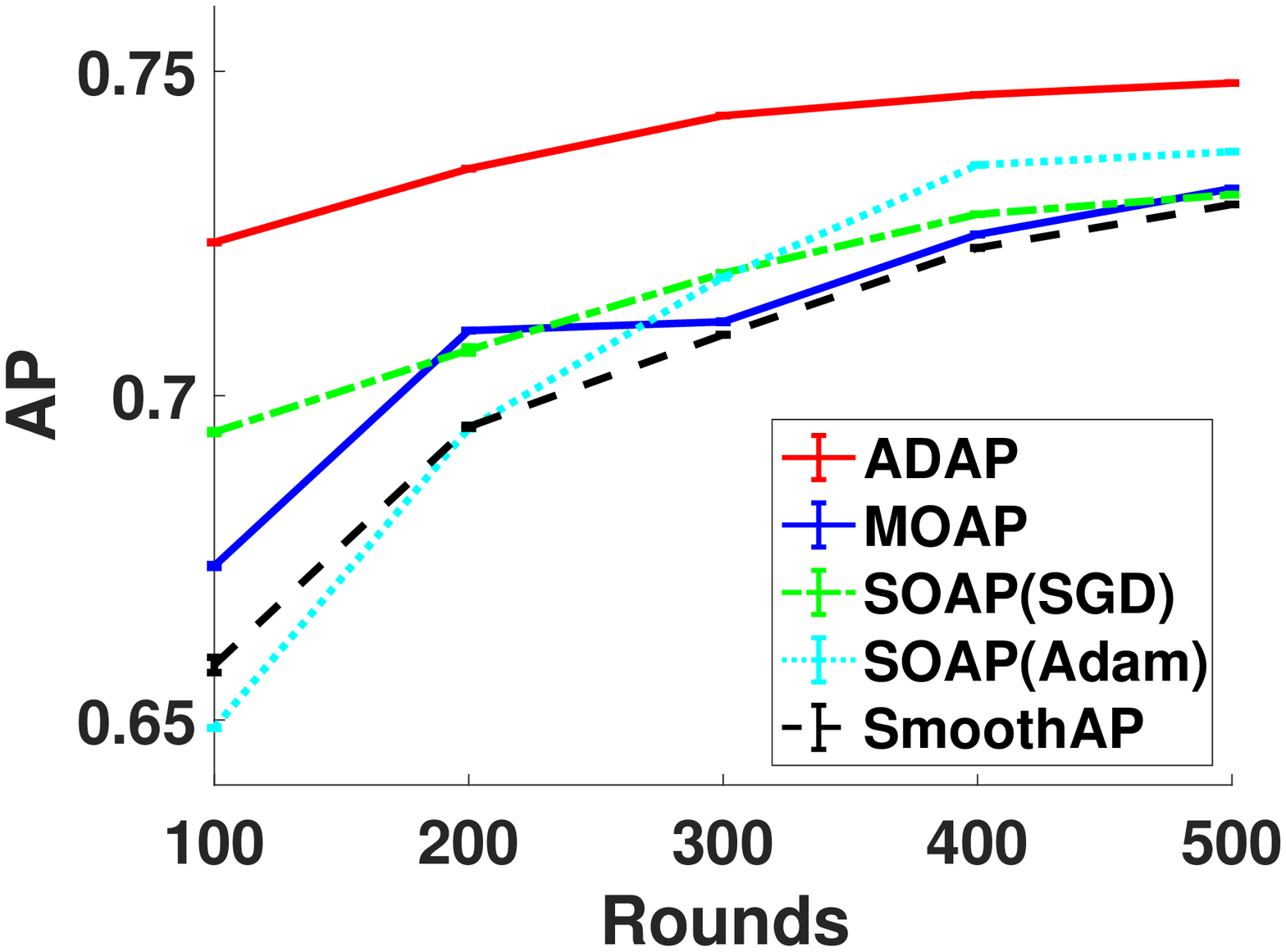}}
\subfigure[w8a dataset]{
\label{Fig.sub.1}
\includegraphics[width=0.31\textwidth]{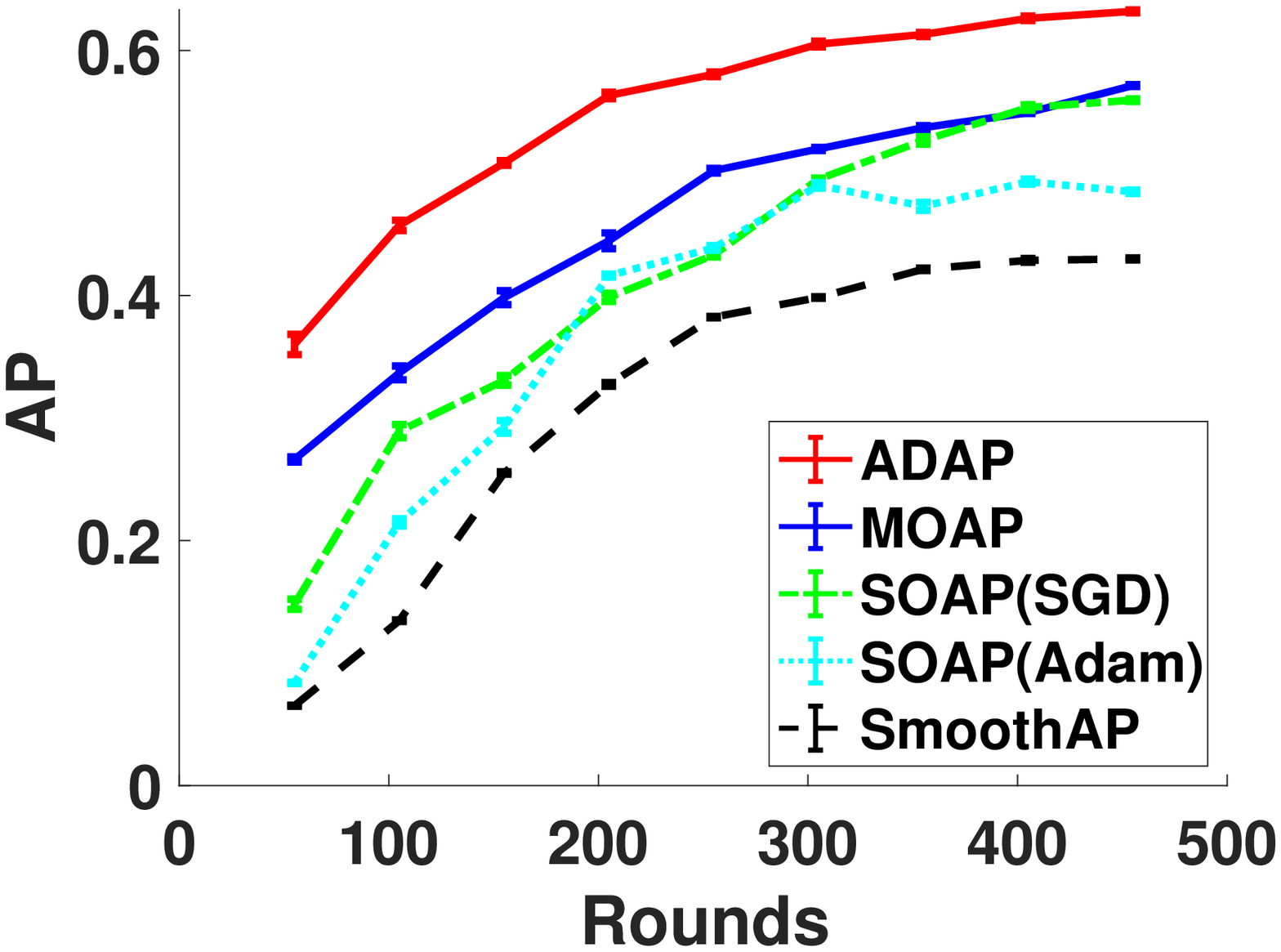}} 
\subfigure[ijcnn1 dataset]{
\label{Fig.sub.6}
\includegraphics[width=0.31\textwidth]{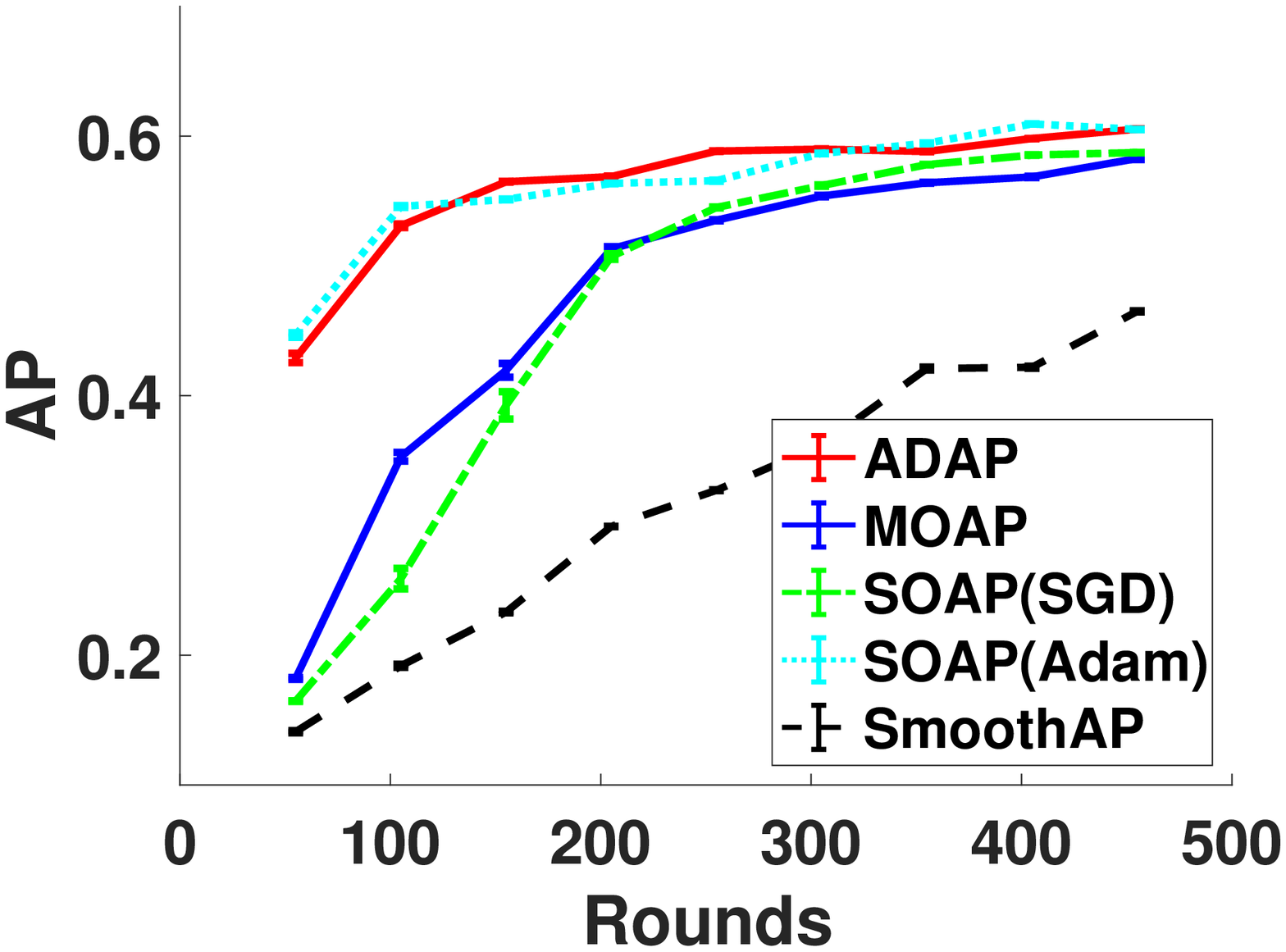}} 
\caption{AP vs \# of rounds  on the training set}
\label{Fig.main}
\end{figure*}

\paragraph{Results}. 
The convergence curves of AP on training examples are reported in Figure \ref{Fig.main}, and the final AP scores on the testing data are shown in Table \ref{exp:tab:3}. We also plot the Precision-Recall curves of the final models on testing data in Figure~\ref{Fig.main2}.  From Figure \ref{Fig.main}, we can see that the proposed ADAP converges faster than other methods. MOAP converges faster than SOAP on mushrooms, phishing, and w8a, and has similar performance as SOAP on other three datasets. From Table  \ref{exp:tab:3}, we can see that ADAP has the best performance on all data sets except ijcnn1, on which ADAP is similar to SOAP (Adam). These results verify the effectiveness of the proposed methods, in particular the adaptive method ADAP. 

\begin{table}
	\caption{Test AUPRC on CIFAR10 and CIFAR100.}	
	\label{exp:tab:4} 
	\centering
	{\begin{tabular}{l|c|c}
			\toprule
		 Method	&CIFAR10&CIFAR100\\
		\hline
	    ADAP &$\mathbf{0.7667\pm0.0015}$&$\mathbf{0.6371\pm 0.0041}$\\
	    SOAP (Adam) &$0.7629\pm 0.0014$&$0.6251\pm 0.0053$\\
	    SmoothAP& $0.7365\pm0.0088$&$0.6071\pm0.0064$ \\
		\bottomrule 
	\end{tabular}
}
\end{table}

\subsection{Training Deep Neural Networks}

Next, we present empirical results on optimizing deep neural networks. We mainly focus on comparing the performance of our proposed ADAP (Adam-style) algorithm with the two state-of-the-art algorithms, i.e., SmoothAP~\citep{DBLP:conf/eccv/BrownXKZ20} and SOAP (Adam)~\citep{qi2021stochastic}. 

\paragraph{Datasets.} Following~\cite{qi2021stochastic}, we conduct experiments on the imbalanced binary-class version of CIFAR10 and CIFAR100 datasets, which are constructed as follows:
Firstly, half of the classes in the original CIFAR10 and CIFAR100 datasets are designated to be the positive class, and the rest half of classes are considered to be the negative class. Then, we remove $98\%$ of the positive examples in the training set to make it imbalanced, while keeping the test set unchanged. Finally, the training data set is splited into train/validation sets as a $4:1$ ratio. 

\paragraph{Configurations.} We choose ResNet-18~\citep{he2016deep} to be the neural network for our imbalanced binary image  classification task. Before the training process, similar to~\citep{qi2021stochastic}, the ResNet-18 model is initialized with a model pretrained by the Adam optimizer for optimizing the cross entropy loss, whose learning rate and weight decay parameter are searched in $\{10^{-6},10^{-5},10^{-4}\}$. Then, the last layer of the neural network is  re-initialized and the network is trained by different AP maximization methods, with hyper-parameters individually tuned for fair comparison. For SOAP and ADAP, we tune the $\beta_0$ parameter in  range
$\{0.001,0.1,0.5,0.9, 0.99,0.999\}$, the margin parameter $m$ in $\{0.5, 1, 2, 5, 10\}$, and learning rate $\eta$ in $\{10^{-6},10^{-5},10^{-4}\}$.

\paragraph{Results} We conduct each experiment for 5 times, and report the mean test AUPRC as well as the standard variation in Table \ref{exp:tab:4}. We could observe that our proposed ADAP enjoys the best results on both datasets.
\section{CONCLUSION AND FUTURE WORK}
\label{futurework}
In this paper, we investigate the stochastic  optimization of AUPRC by maximizing a surrogate loss of AP, which is one of the most important performance metrics for imbalanced classification problems.    
Previous study has shown that an $O(1/\epsilon^5)$ convergence rate can be achieved. In this work, we further improve the theoretical guarantee by proposing a stochastic momentum method as well as a family of adaptive methods, which enjoy a better convergence rate of $O(1/\epsilon^4)$. Our essential ideas the  are two-folds: (i) updating the biased estimator for individual   ranking-scores-tracking in a randomized coordinate-wise manner; (ii)  applying the momentum step on top of the stochastic gradient estimator for tracking the gradient of the objective. 
Empirical studies on optimizing linear models and deep networks show the effectiveness of the proposed adaptive methods.  For future work, we consider applying the proposed algorithms to other deep learning tasks. 
\subsubsection*{Acknowledgements}
T. Yang was was partially supported by
NSF Career Award $\#$1844403, NSF Award $\#$2110545, NSF Award $\#$1933212. L. Zhang was partially supported by JiangsuSF (BK20200064).


\bibliographystyle{plainnat}
\bibliography{plain}
\newpage

\appendix

\section{NOTATIONS}
We first introduce the following lemma, which demonstrates the basic  properties of the functions in our optimization problem \eqref{eqn:optimization:problem}. 
\begin{lemma}
\label{lem:bounded:function:2}
Suppose Assumptions \ref{ass:bounded:ell} and \ref{ass:bounded:g} hold, then there exist $G, C_g, L_g, C_f, L_f, C_F, L_F>0$, such that $\forall \w,$ $\|g_i(\w)\|^2\leq G^2$, $g_i(\w)$ is $C_g$-Lipschitz and $L_g$-smooth, and $\forall \u\in\Omega,$ $f(\u)$ is $C_f$-Lipschitz and $L_f$-smooth. Finally, $\forall \w$, $F(\w)$ is $C_F$-Lipschiz and $L_F$-smooth. 
\end{lemma}
\begin{proof}
For $g_i(\w)$, we have $\| g_i(\w)\|^2\leq m^2M^2+n^2M^2=G^2$,
 $\|\nabla g_i(\w)\|^2\leq m^2\sigma^2 + n^2\sigma^2 = C_g^2$, and $\|\nabla g_i(\w) - \nabla g_i(\w')\|^2\leq m^2L_l^2\|\w - \w'\|^2 + n^2L_l^2\|\w- \w'\|^2 = L_g^2 \|\w - \w'\|^2$. Similarly, for $f(\u)$, we have  $\|\nabla f(\u)\|_2^2 =\| (\frac{-1}{u_2},  \frac{u_1}{u_2^2})\|^2\leq 1/C^2+  M^2m^2/C^2 = C_f^2$ and $\|\nabla f(\u) - \nabla f(\u')\|_2^2 =\| (\frac{-1}{u_2} - \frac{-1}{u_2'},  \frac{u_1}{u_2^2} - \frac{u_1'}{u_2^{'2}}) \|^2\leq \frac{(u_2 - u_2')^2}{C^2}+  \frac{2(u_1 - u_1')^2}{C^4} + \frac{8(Mm)^2(Mn)^2(u_2 - u_2')^2}{C^8}\leq L_f^2 (\|\u - \u'\|^2)$ for any $\u, \u'\in\Omega$. Hence, $f(g_i(\w))$ is $C_F=(C_gC_f)$-Lipschitz, and 
  $L_F= (C_g^2L_f + L_gC_f)$-smooth. Since $F(\w)$ is the average of $f(g_i(\w))$, it is also $C_F$-Lipschitz and $L_F$-smooth. 
\end{proof}

\section{PROOF OF THEOREM \ref{thm:msam:v1}}
Let $\Delta_t=\m_{t+1}-\nabla F(\w_t).$ According to the update rule in Algorithm \ref{alg:AP-MSGD-V1}, and we have
\begin{equation}
\begin{split}
\label{eqn:appendix:theorem:1}
\|\Delta_t\|^2=&\left\|  (1-\beta_1)\m_t+\beta_1 \frac{1}{m}\sum_{i=1}^m  \nabla\tilde{g}_{i}(\w_t)^{\top}\cdot \nabla f\left([U_{t+1}]^{\top}_i\right)-\nabla F(\w_t)\right\|^2\\
=&\Bigg\| \underbrace{(1-\beta_1) [\m_t-\nabla F(\w_{t-1})]}_{A_1} +\underbrace{(1-\beta_1) [\nabla  F(\w_{t-1})-\nabla  F(\w_t)]}_{A_2}\\
& +\underbrace{\beta_1 \left[ \frac{1}{m}\sum_{i=1}^m  \nabla \tilde{g}_{i}(\w_t)^{\top}\cdot \nabla f\left([U_{t+1}]^{\top}_i\right)- \frac{1}{m}\sum_{i=1}^m  \nabla \tilde{g}_{i}(\w_t)^{\top}\cdot \nabla f\left(g_i(\w_t)\right)\right] }_{A_3} \\
&+\underbrace{\beta_1 \left[ \frac{1}{m}\sum_{i=1}^m  \nabla \tilde{g}_{i}(\w_t)^{\top}\cdot \nabla f\left(g_i(\w_t)\right) - \nabla F(\w_t) \right]}_{A_4} \Bigg\|^2.
\end{split}	
\end{equation}
Let $\E_t[\cdot]=\E[\cdot|\F_{t-1}]$, where $\F_{t-1}$ is the $\sigma$-algebra of MOAP-V1, which contains the learning history from round 1 to round $t-1$. Since $\tilde{g}_t(\w)$ is an unbiased estimator of $g_i(\w)$, we have $\E_t[A_4]=0$. Moreover, $A_1$ and $A_2$ are fixed given $\F_{t-1}$. Thus
\begin{equation}
\begin{split}
\label{eqn:appendix:theorem:A1}
{} &\E_t[\| \Delta_t\|^2]	\\
\leq {} &  \E_t[\| A_1+A_2+A_3\|^2] + \E_t[\|A_4\|^2]+2\E_t[\|A_3\|\|A_4\|]+\underbrace{2(A_1+A_2)^{\top}\E_t[A_4]}_{=0}\\   
\leq {}  & \E_t[\|A_1+A_2+A_3\|^2]+2\E_t[\|A_4\|^2]+\E_t[\|A_3\|^2]\\
\leq {} & (1+\beta_1)\E_t[\|A_1\|^2]+\left(1+\frac{1}{\beta_1}\right)\E_t[\|A_2+A_3\|^2]+2\E_t[\|A_4\|^2]+\E_t[\|A_3\|^2]\\
\leq {} & {(1+\beta_1^2)}(1-\beta_1)\|\Delta_{t-1}\|^2+\frac{2(1+\beta_1)}{\beta_1}\|A_2\|^2+\frac{2+3\beta_1}{\beta_1}\E_t[\|A_3\|^2]+2\E_t[\|A_4\|^2],
\end{split}	
\end{equation}
where the third inequality is based on Young's inequality. To proceed, according to Lemma \ref{lem:bounded:function:2}, we have 
\begin{equation}
\label{eqn:appendix:theorem:A2}
\|A_2\|^2=(1-\beta_1)^2\| \nabla  F(\w_{t-1})-\nabla  F(\w_t)\|^2 \leq (1-\beta_1)^2 L^2_F \|\w_t-\w_{t-1}\|^2,
\end{equation}
\begin{equation}
\begin{split}
\label{eqn:appendix:theorem:A3}
	\|A_3\|^2 \leq {} & \beta_1^2 \frac{1}{m}\sum_{i=1}^m\|  \nabla \tilde{g}_{i}(\w_t)^{\top}\cdot \nabla f\left([U_{t+1}]^{\top}_i\right)-  \nabla \tilde{g}_{i}(\w_t)^{\top}\cdot \nabla f\left(g_i(\w_t)\right)\|^2\\
	\leq {} & \beta_1^2 \frac{C_g^2}{m}\sum_{i=1}^m \| \nabla f([U_{t+1}]^{\top}_i)-\nabla f(g_i(\w_t)) \|^2\\
	\leq {} & \beta_1^2 \frac{C^2_gL^2_f}{m} \sum_{i=1}^m \| [U_{t+1}]_i^{\top} -g_i(\w_t)\|^2
	= {}  \frac{\beta_1^2C_g^2L^2_f}{m}\| U_{t+1}-g(\w_t)\|^2,
\end{split}
\end{equation}
where $g(\w_t)=[g_1(\w_t),\dots, g_m(\w_t)]^{\top}$, and 
 \begin{equation}
 \label{eqn:appendix:theorem:A4}
 \E_t[\|A_4\|^2]\leq {} 	\beta_1^2 \E_t\left[ \left\| \frac{1}{m}\sum_{i=1}^m \nabla  \tilde{g}_{i}(\w_t)^{\top}\cdot \nabla f([U_{t+1}]_i^{\top})   \right\|^2 \right]\leq \beta_1^2 C_g^2C_f^2.
 \end{equation}
 Combining 
 \eqref{eqn:appendix:theorem:A1},
 \eqref{eqn:appendix:theorem:A2}, \eqref{eqn:appendix:theorem:A3},  \eqref{eqn:appendix:theorem:A4}, we get 
 \begin{equation}
 	\begin{split}
 \E_t[\| \Delta_t\|^2]\leq {} &  (1-\beta_1)\|\Delta_{t-1}\|^2 + \frac{2L_F^2\|\w_t-\w_{t-1}\|^2}{\beta_1}+\frac{5\beta_1 C_g^2L^2_f\E_t[\|U_{t+1}-g(\w_{t})\|^2]}{m}\\
 {} &  + 2\beta_1^2  C_g^2C_f^2.
 	\end{split}
 \end{equation}
 Summing over from 1 to $T$, we get 
\begin{equation}
\begin{split}
\E\left[\sum_{i=1}^T\|\Delta_i\|^2\right]\leq & \frac{\|\Delta_0\|^2}{\beta_1}+\frac{2L_F^2\sum_{i=1}^T\|\w_i-\w_{i-1}\|^2}{\beta_1^2}+\frac{5C_g^2L^2_f
\sum_{i=1}^T\E_i[\|U_{i+1}-g(\w_{i})\|^2]}{m} \\
&+ 2\beta_1  C_g^2C_f^2T\\
= & \frac{2C_F^2}{\beta_1}+\frac{2L_F^2\eta^2\sum_{i=1}^T\|\m_i\|^2}{\beta^2}+\frac{5C_g^2L^2_f
\sum_{i=1}^T\E_i[\|U_{i+1}-g(\w_{i})\|^2]}{m} \\
&+ 2\beta_1  C_g^2C_f^2T
\end{split}	
\end{equation}
Note that here $\w_0$ can be considered as a psudo-decision which follows the same update procedure as $\w_t$. Next, we turn to bound the third term of the R.H.S. of the inequality, and introduce the following lemma. 
\begin{lemma} [Variance Recursion~\cite{,wang2017stochastic}] \label{lem:variance:recursion}
Consider a sequence $\m_{t+1}=(1-\beta)\m_t+\beta \tilde{h}(\w_t)$ for tracking $h(\w_t)$, where $\E[\tilde{h}(\w_t)]=h(\w_t)$ and $h$ is a $C$-Lipchitz continuous mapping. Then we have 
\begin{equation}
\E_t[\| \m_{t+1}-h(\w_t)\|^2]\leq (1-\beta) \| \m_{t}-h(\w_{t-1})\|^2+\beta^2 \E_t[\|\tilde{h}(\w_t)-h(\w_t)\|^2]+\frac{2C^2\|\w_t-\w_{t-1}\|^2}{\beta}.
\end{equation}
\end{lemma}
With Lemma \ref{lem:variance:recursion}, we have 
\begin{equation}
\label{eqn:appendix:theorem:1:2}
\sum_{i=1}^T \E_i[\|U_{i+1}-g(\w_i)\|^2]\leq \frac{mG^2}{\beta_0}	+2\beta_0 T mG^2+\frac{2 mC^2_{g}\eta^2\sum_{i=1}^T\|\m_i\|^2}{\beta_0^2}.
\end{equation}
Combining all the about results together, and set $\beta_0=\beta_1=\beta$, we have
\begin{equation}
\begin{split}
\E\left[\sum_{i=1}^T \| \Delta_i\|^2\right]\leq {} & \frac{2C_F^2+5C_g^2L_f^2G^2}{\beta}+\frac{(2L_F^2+10C_g^4L_f^2)\eta^2}{\beta^2}\sum_{i=1}^T\|\m_i\|^2 \\
{} &+	\beta(10C_g^2L_f^2G^2+2C_g^2C_f^2)T.
\end{split}	
\end{equation}
We introduce the following lemma \cite{guo2021stochastic}, which can be obtained following the definition of the smoothness.  
\begin{lemma}
\label{Lemma:smooth}
Consider a sequence $\w_{t+1}=\w_t-\eta \m_{t+1}$ for a $L_F$-smooth function $F$, with $\eta L_F\leq 1/2$ we have 
$$F(\w_{t+1})\leq F(\w_t)+\frac{\eta}{2}\|\nabla F(\w_t)-\m_{t+1}\|^2-\frac{\eta}{2}\| \nabla F(\w_t)\|^2-\frac{\eta}{4}\|\m_{t+1}\|^2.$$	
\end{lemma}
Based on Lemma \ref{Lemma:smooth}, with $(2L_F^2+10C_g^4L_f^2)\eta^2/\beta^2\leq 1/2$, we get 
\begin{equation}
\begin{split}
	\E\left[\frac{1}{T}\sum_{t=1}^T\|\nabla F_{\w}(\w_t)\|^2\right]\leq {} &\frac{2( F(\w_{T+1})-F(\w_1))}{\eta T}+ \frac{C_F^2+5C_g^2L_f^2G^2}{\beta T}\\
	{} & +\beta(10C_g^2L_f^2G^2+2C_g^2C_f^2).
\end{split}	
\end{equation}
Thus, by setting $\beta=O(\epsilon^2)$ and $\eta=O(\beta)$, we have 
$\E\left[\frac{1}{T}\sum_{t=1}^T\|\nabla F_{\w}(\w_t)\|^2\right]=O(\epsilon^2)$ for $$T\geq \max \left\{\frac{2(F(\w_{1})-F(\w_{T+1}))}{\eta \epsilon^2},\frac{C_F^2+10C_g^2L_f^2G^2}{\beta\epsilon^2}\right\}.$$
\section{PROOF OF THEOREM \ref{thm:msam:v2}}
Let $i_t$ be the positive data chosen in round $t$, and for convenience we set $\beta_0=\beta_1=\beta$. According to the update rule in Algorithm \ref{alg:AP-MSGD-V1}, and we have
\begin{equation}
\begin{split}
\|\Delta_t\|^2=&\left\|  (1-\beta)\m_t+\beta  \nabla\tilde{g}_{i_t}(\w_t)^{\top}\cdot \nabla f\left([U_{t+1}]^{\top}_{i_t}\right)-\nabla F(\w_t)\right\|^2\\
=&\Bigg\| \underbrace{(1-\beta) [\m_t-\nabla F(\w_{t-1})]}_{A_1} +\underbrace{(1-\beta) [\nabla  F(\w_{t-1})-\nabla  F(\w_t)]}_{A_2}\\
& +\underbrace{\beta \left[\nabla \tilde{g}_{i_t}(\w_t)^{\top}\cdot \nabla f\left([U_{t+1}]^{\top}_{i_t}\right)- \nabla \tilde{g}_{i_t}(\w_t)^{\top}\cdot \nabla f\left(g_{i_t}(\w_t)\right)\right] }_{A_3} \\
&+\underbrace{\beta \left[ \nabla \tilde{g}_{i_t}(\w_t)^{\top}\cdot \nabla f\left(g_{i_t}(\w_t)\right) - \nabla F(\w_t) \right]}_{A_4} \Bigg\|^2.
\end{split}	
\end{equation}
Thus 
\begin{equation}
\begin{split}
{} &\E_t[\| \Delta_t\|^2]	
\leq (1-\beta)\|\Delta_{t-1}\|^2+\frac{2(1+\beta)}{\beta}\|A_2\|^2+\frac{2+3\beta}{\beta}\E_t[\|A_3\|^2]+2\E_t[\|A_4\|^2],
\end{split}	
\end{equation}
where 
\begin{equation}
\|A_2\|^2=(1-\beta_1)^2\| \nabla  F(\w_{t-1})-\nabla  F(\w_t)\|^2 \leq (1-\beta_1)^2 L^2_F \|\w_t-\w_{t-1}\|^2,
\end{equation}
\begin{equation}
\begin{split}
	\|A_3\|^2 \leq {} & \beta^2 \|  \nabla \tilde{g}_{i_t}(\w_t)^{\top}\cdot \nabla f\left([U_{t+1}]^{\top}_{i_t}\right)-  \nabla \tilde{g}_{i_t}(\w_t)^{\top}\cdot \nabla f\left(g_{i_t}(\w_t)\right)\|^2\\
	\leq {} & \beta^2 {C_g^2}\| \nabla f([U_{t+1}]^{\top}_{i_t})-\nabla f(g_{i_t}(\w_t)) \|^2\\
	\leq {} & \beta^2 {C^2_gL^2_f}  \| [U_{t+1}]_{i_t}^{\top} -g_{i_t}(\w_t)\|^2,
\end{split}
\end{equation}
and 
 \begin{equation}
\E[ \|A_4\|^2]\leq {} 	\beta^2  \E[\| \nabla  \tilde{g}_{i_t}(\w_t)^{\top}\cdot \nabla f([U_{t+1}]_{i_t}^{\top})   \|^2]\leq \beta_1^2 C_g^2C_f^2.
 \end{equation}
Combining the results above, we get 
 \begin{equation}
 	\begin{split}
 	\label{eqn:th2:35}
 \E[\| \Delta_t\|^2]\leq {} &  (1-\beta)\|\Delta_{t-1}\|^2 + \frac{2L_F^2\|\w_t-\w_{t-1}\|^2}{\beta_1}+{5\beta C_g^2L^2_f\E[\|[U_{t+1}]_{i_t}-g_{i_t}(\w_{t})\|^2]}\\
 {} &  + 2\beta^2  C_g^2C_f^2.
 	\end{split}
 \end{equation}
Next, we turn to bound the third term on the R.H.S. of \eqref{eqn:th2:35}. Following \cite{qi2021stochastic}, we divide the whole interval into $m$ groups, with the $i$-th group $\T_i=\{t_1^i,\dots, t_k^i,\dots\}$, where $t_k^i$ denotes the $k$-th time that the $i$-th positive data is chosen to update $[U_{t_k^i}]_i$ and obtain $[U_{t_k^i+1}]_i$. Note that since we only update the chosen data, we have $[U_{t_k^i+1}]_i=\dots=[U_{t_{k+1}^i}]_i$. Without loss of generality, assume $i_t$ is picked for the $k$-th time at round $t$. Thus
\begin{equation}
\begin{split}
\label{eqn:th2:356}
& \E[\|[U_{t+1}]_{i_t}-g_{i_t}(\w_{t})\|^2] \\
={} &  \E[\|[U_{t_{k}^{i_t}+1}]_{i_t}-g_{i_t}(\w_{t_{k}^{i_t}})\|^2]  \\
	\leq {} &\E[\|(1-\beta)[U_{t^{i_t}_{k-1}+1}]_{i_t}+\beta \tilde{g}_i(\w_{t^{i_t}_{k}})-g_{i_t}(\w_{t_{k}^{i_t}})\|^2]\\
	= {} & \E[\|(1-\beta)([U_{t_{k-1}^{i_t}+1}]_{i_t}-g_{i_t}(\w_{t_{k-1}^{i_t}}) )+ \beta (\tilde{g}_{i_t}(\w_{t_{k-1}^{i_t}})-g_{i_t}(\w_{t_{k-1}^{i_t}})) + g_{i_t}(\w_{t_{k-1}^{i_t}})-g_{i_t}(\w_{t_{k}^{i_t}})\|^2 ]\\
		\leq {} & (1+\beta)(1-\beta)^2\E[\|[U_{t_{k-1}^{i_t}+1}]_{i_t}-g_{i_t}(\w_{t_{k-1}^{i_t}})\|^2] + 2(1+1/\beta)\beta^2\cdot 2G^2 \\
{} &	+ 2(1+1/\beta) \E[\| g_{i_t}(\w_{t_{k-1}^{i_t}})-g_{i_t}(\w_{t_{k}^{i_t}})\|^2]\\
\leq {} & (1-\beta) \E[\|[U_{t_{k-1}^{i_t}+1}]_{i_t}-g_{i_t}(\w_{t_{k-1}^{i_t}})\|^2]+ 8\beta G^2+\frac{4C_g^2}{\beta} \E[\| \w_{t_{k-1}^{i_t}}-\w_{t_{k}^{i_t}}\|^2]\\
\leq {} &  (1-\beta) \E[\|[U_{t_{k-1}^{i_t}+1}]_{i_t}-g_{i_t}(\w_{t_{k-1}^{i_t}})\|^2]+ 8\beta G^2 + \frac{4C_g^2\eta^2}{\beta} \E \left[ \left\| \sum_{j=t^{i_t}_{k-1}+1}^{t_k^{i_t}} \m_t\right\|^2\right]\\
\leq {} &  (1-\beta) \E[\|[U_{t_{k-1}^{i_t}+1}]_{i_t}-g_{i_t}(\w_{t_{k-1}^{i_t}})\|^2]+ 8\beta G^2 + \frac{4C_g^2\eta^2}{\beta} \E\left[ (t_{k}^{i_t}-t_{k-1}^{i_t}) \sum_{j=t^{i_t}_{k-1}+1}^{t_k^{i_t}} \left\|  \m_j\right\|^2\right]
\end{split}	
\end{equation}
where the second and last inequalities are based on Young's inequality. Note that, based on Lemma 1 of \cite{qi2021stochastic}, $t_{k}^{i_t}-t_{k-1}^{i_t}$ is a random variable with conditional distribution given by a geometric distribution with $p=1/m$, i.e., $\E[t_{k}^{i_t}-t_{k-1}^{i_t}|t_{k-1}^{i_t}]\leq m$. However, since $t_{k}^{i_t}-t_{k-1}^{i_t}$ and $\sum_{j=t^{i_t}_{k-1}+1}^{t_k^{i_t}} \left\|  \m_j\right\|^2$ are \emph{dependent}, this conclusion can not be applied to bound the last term, which further making the advantages of the momentum can not be used. Because of this issue, we have to rewrite \eqref{eqn:th2:356} and bound the last term as 
\begin{equation}
\begin{split}
	\E\left[ \left\| \sum_{j=t^{i_t}_{k-1}+1}^{t_k^{i_t}} \m_t\right\|^2\right]\leq C_F^2\E[(t_{k}^{i_t}-t_{k-1}^{i_t})^2]\leq 2m^2C_F^2.
\end{split}	
\end{equation}
Thus 
 $$  \E[\|[U_{t_{k}^{i_t}+1}]_{i_t}-g_{i_t}(\w_{t_{k}^{i_t}})\|^2]\leq (1-\beta) \E[\|[U_{t_{k-1}^{i_t}}]_{i_t}-g_{i_t}(\w_{t_{k-1}^{i_t}})\|^2]+ 8\beta G^2 + \frac{8m^2C_F^2C_g^2\eta^2}{\beta}.$$
Summing over for all intervals, we have
$$ \sum_{i=1}^m \E[\sum_{k=1}^{K_i}\|[U_{t_{k-1}^{i}}]_{i}-g_{i}(\w_{t_{k-1}^{i}})\|^2]\leq \frac{mG^2}{\beta}+8 G^2 T + \frac{8m^2C_F^2C_g^2\eta^2T}{\beta^2}$$
where $K_i$ denotes that the $i$-th positive data is chosen for $K_i$ times.
Combining with \eqref{eqn:th2:35}, we get
\begin{equation}
\begin{split}
\E[\sum_{i=1}^T \|\Delta_i\|^2]\leq  \frac{2L_f^2\eta^2C_F^2T}{\beta^2} +\frac{40C_g^4L_f^2m^2C_F^2\eta^2T}{\beta}+
\frac{C_F^2}{\beta}+5C_g^2L_f^2mG^2+\beta T (40C_G^2L_f^2G^2+8C_g^2C_f^2).
\end{split}	
\end{equation}

With Lemma \ref{Lemma:smooth}, we obtain 
\begin{equation}
\begin{split}
\E\left[\frac{1}{T}\sum_{i=1}^T \|\nabla F(\w) \|^2\right]\leq {} &\frac{2(F(\w_1)-F(\w_{T+1}))}{\eta T} + 	\frac{2L_f^2\eta^2C_F^2}{\beta^2} +\frac{40C_g^4L_f^2m^2C_F^2\eta^2}{\beta}\\
{} &+ \frac{C_F^2}{\beta T}+\frac{5C_g^2L_f^2mG^2}{T}+\beta  (40C_G^2L_f^2G^2+8C_g^2C_f^2).
\end{split}
\end{equation}
Finally, we can obtain a $\epsilon$-stationary solution by setting $\beta=O(\epsilon^2 )$, $\eta=O(\epsilon^3/m)$, and 
$$T\geq \left\{ \frac{2(F(\w_1)-F(\w_{T+1}))}{\eta \epsilon^2} , \frac{C_F^2}{\beta \epsilon^2},\frac{5C_g^2L_f^2mG^2}{\epsilon^2} \right\}.$$
\section{PROOF OF THEOREM \ref{thm:3}}
\label{app:tho:3}
Let $\Delta_t=\m_{t+1}-\nabla F(\w_t).$ According to the update rule in Algorithm \ref{alg:AP-MSGD-V2}, and we have
\begin{equation}
\begin{split}
\label{eqn:appendix:theorem:1:2}
\|\Delta_t\|^2=&\left\|  (1-\beta_1)\m_t+\beta_1 \frac{1}{B}\sum_{\x_i\in\B_t}  \nabla\tilde{g}_{i}(\w_t)^{\top}\cdot \nabla f\left([U_{t+1}]^{\top}_i\right)-\nabla F(\w_t)\right\|^2\\
=&\Bigg\| \underbrace{(1-\beta_1) [\m_t-\nabla F(\w_{t-1})]}_{A_1} +\underbrace{(1-\beta_1) [\nabla  F(\w_{t-1})-\nabla  F(\w_t)]}_{A_2}\\
& +\underbrace{\beta_1 \left[ \frac{1}{B}\sum_{\x_i\in\B_t}  \nabla \tilde{g}_{i}(\w_t)^{\top}\cdot \nabla f\left([U_{t+1}]^{\top}_i\right)- \frac{1}{B}\sum_{\x_i\in\B_t}  \nabla \tilde{g}_{i}(\w_t)^{\top}\cdot \nabla f\left(g_i(\w_t)\right)\right] }_{A_3} \\
&+\underbrace{\beta_1 \left[ \frac{1}{B}\sum_{\x_i\in\B_t}  \nabla \tilde{g}_{i}(\w_t)^{\top}\cdot \nabla f\left(g_i(\w_t)\right) - \nabla F(\w_t) \right]}_{A_4} \Bigg\|^2.
\end{split}	
\end{equation}
Following similar procedure as in the proof of Theorem \ref{thm:msam:v1}, we have 
\begin{equation}
\begin{split}
\label{eqn:appendix:theorem:A1:3}
{} &\E_t[\| \Delta_t\|^2]	
\leq (1-\beta_1)\|\Delta_{t-1}\|^2+\frac{2(1+\beta_1)}{\beta_1}\|A_2\|^2+\frac{2+3\beta_1}{\beta_1}\E_t[\|A_3\|^2]+2\E_t[\|A_4\|^2],
\end{split}	
\end{equation}
where 
\begin{equation}
\label{eqn:appendix:theorem:A2:3}
\|A_2\|^2=(1-\beta_1)^2\| \nabla  F(\w_{t-1})-\nabla  F(\w_t)\|^2 \leq (1-\beta_1)^2 L^2_F \|\w_t-\w_{t-1}\|^2
\end{equation}
\begin{equation}
\begin{split}
\label{eqn:appendix:theorem:A3:3}
	\|A_3\|^2 \leq {} & \beta_1^2 \frac{1}{B}\sum_{\x_i\in\B_t}\|  \nabla \tilde{g}_{i}(\w_t)^{\top}\cdot \nabla f\left([U_{t+1}]^{\top}_i\right)-  \nabla \tilde{g}_{i}(\w_t)^{\top}\cdot \nabla f\left(g_i(\w_t)\right)\|^2\\
	\leq {} & \beta_1^2 \frac{C_g^2}{B}\sum_{\x_i\in\B_t} \| \nabla f([U_{t+1}]^{\top}_i)-\nabla f(g_i(\w_t)) \|^2\\
	\leq {} & \beta_1^2 \frac{C^2_gL^2_f}{B} \sum_{\x_i\in\B_t} \| [U_{t+1}]_i^{\top} -g_i(\w_t)\|^2
\end{split}
\end{equation}
and 
 \begin{equation}
 \label{eqn:appendix:theorem:A4:3}
 \E_t[\|A_4\|^2]\leq {} 	\beta_1^2 \E_t\left[\left\| \frac{1}{B}\sum_{\x_i\in\B_t}  \nabla  \tilde{g}_{i}(\w_t)^{\top}\cdot \nabla f([U_{t+1}]_i^{\top}) \right\|^2\right] \leq \beta_1^2 C_g^2C_f^2.
 \end{equation}
 Combining 
 \eqref{eqn:appendix:theorem:A1:3},
 \eqref{eqn:appendix:theorem:A2:3}, \eqref{eqn:appendix:theorem:A3:3},  \eqref{eqn:appendix:theorem:A4:3}, and taking expectation over the randomness of $\B_t$, we get 
 \begin{equation}
 	\begin{split}
 \E_t[\| \Delta_t\|^2]\leq {} &  (1-\beta_1)\|\Delta_{t-1}\|^2 + \frac{2L_F^2\|\w_t-\w_{t-1}\|^2}{\beta_1}+\frac{5\beta_1 C_g^2L^2_f\E_t[\|U_{t+1}-g(\w_{t})\|^2]}{m}\\
 {} &  + 2\beta_1^2  C_g^2C_f^2.
 	\end{split}
 \end{equation}
 Summing over from 1 to $T$ and taking the expectation with respect to all the randomness, we get 
\begin{equation}
\begin{split}
\E\left[\sum_{i=1}^T\|\Delta_i\|^2\right]\leq & \frac{\|\Delta_0\|^2}{\beta_1}+\frac{2L_F^2\sum_{i=1}^T\|\w_i-\w_{i-1}\|^2}{\beta_1^2}+\frac{5C_g^2L^2_f
\sum_{i=1}^T\E_i[\|U_{i+1}-g(\w_{i})\|^2]}{m} \\
&+ 2\beta_1  C_g^2C_f^2T\\
= & \frac{C_F^2}{\beta_1}+\frac{2L_F^2\eta^2\sum_{i=1}^T\|\m_i\|^2}{\beta^2}+\frac{5C_g^2L^2_f
\sum_{i=1}^T\E_i[\|U_{i+1}-g(\w_{i})\|^2]}{m} \\
&+ 2\beta_1  C_g^2C_f^2T.
\end{split}	
\end{equation}

Next, note that $U_{t+1}=\Pi_{\Omega^m}[(1-\beta_0)U_t + \beta_0 \widehat{g}(\w_t)]$, where $\widehat{g}(\w_t)$ is an unbiased estimator of $g(\w_t)$, which is defined in \eqref{eqn:cord}. Thus, with Lemma \ref{lem:variance:recursion} and the fact that 
$$\E[\|\widehat{g}(\w_t)-g(\w_t)\|^2]\leq \E[\|\widehat{g}(\w_t)\|^2]= \E\left[\sum_{\x_i\in \B_t}\left\| \frac{m}{B}\tilde{g}_i(\w)\right\|^2\right]\leq \frac{m^2}{B}G^2.$$
we have 
\begin{equation}
\label{eqn:appendix:theorem:1:2:3}
\sum_{i=1}^T \E_i[\|U_{i+1}-g(\w_i)\|^2]\leq \frac{mG^2}{\beta_0}	+2\beta_0 T \frac{m^2}{B}G^2+\frac{2 mC^2_{g}\eta^2\sum_{i=1}^T\|\m_i\|^2}{\beta_0^2}.
\end{equation}
Combining all the about results together, and set $\beta_0=\beta_1=\beta$, we have
\begin{equation}
\begin{split}
\E\left[\sum_{i=1}^T \| \Delta_i\|^2\right]\leq {} & \frac{C_F^2+5C_g^2L_f^2G^2}{\beta}+\frac{(2L_F^2+10C_g^4L_f^2)\eta^2}{\beta^2}\sum_{i=1}^T\|\m_i\|^2 \\
{} &+	\beta(10C_g^2L_f^2G^2\frac{m}{B}+2C_g^2C_f^2)T.
\end{split}	
\end{equation}
Based on Lemma \ref{Lemma:smooth}, with $(2L_F^2+10C_g^4L_f^2)\eta^2/\beta^2\leq 1/2$, we get 
\begin{equation}
\begin{split}
	\E\left[\frac{1}{T}\sum_{t=1}^T\|\nabla F(\w_t)\|^2\right]\leq {} &\frac{2( F(\w_{1})-F(\w_T))}{\eta T}+ \frac{C_F^2+5C_g^2L_f^2G^2}{\beta T}\\
	{} & +\beta(5C_g^2L_f^2G^2\frac{m}{B}+2C_g^2C_f^2).
\end{split}	
\end{equation}
Thus, by setting $\beta=O(\epsilon^2B/m)$ and $\eta=O(\beta)$, we have 
$\E\left[\frac{1}{T}\sum_{t=1}^T\|\nabla F_{\w}(\w_t)\|^2\right]=O(\epsilon^2)$ for $$T\geq \max \left\{\frac{2(F(\w_{1})-F(\w_{T+1}))}{\eta \epsilon^2},\frac{C_F^2+5C_g^2L_f^2G^2}{\beta\epsilon^2}\right\}.$$

\section{PROOF OF THEOREM 4} 	
Let $\Delta_t=\m_{t+1}-\nabla F(\w_t),$ and $\eta^s_{t} = 1/(\sqrt{\bv_{t+1}} + \delta)$. Based on the boundedness of $\|tilde{\nabla}\|$, we know that $\eta c_l\leq[\eta^s_{t}]_i\leq \eta c_u$ for all $i\in[d]$. 
According to the update rule in Algorithm \ref{alg:Adam}, we have
\begin{equation}
\begin{split}
\label{eqn:appendix:theorem:1:2}
\|\Delta_t\|^2=&\left\|  (1-\beta_1)\m_t+\beta_1 \frac{1}{B}\sum_{\x_i\in\B_t}  \nabla\tilde{g}_{i}(\w_t)^{\top}\cdot \nabla f\left([U_{t+1}]^{\top}_i\right)-\nabla F(\w_t)\right\|^2
\end{split}	
\end{equation}
Following similar procedure as in Appendix \ref{app:tho:3}, we can get
\begin{equation}
\begin{split}
\label{eqn:fin:1}
\E\left[\sum_{i=1}^T \| \Delta_i\|^2\right]\leq {} & \frac{C_F^2+5C_g^2L_f^2G^2}{\beta}+\frac{(2L_F^2+10C_g^4L_f^2)\eta^2c_u^2}{\beta^2}\sum_{i=1}^T\|\m_i\|^2 \\
{} &+	\beta(10C_g^2L_f^2G^2\frac{m}{B}+2C_g^2C_f^2)T.
\end{split}	
\end{equation}
Next, we introduce the following lemma, which is the counterpart of Lemma \ref{Lemma:smooth} for Adam-style algorithms \cite{guo2021stochastic}. 
\begin{lemma}\label{lem:smooth:adam}
For $\w_{t+1} = \w_t- \eta_t^s\cdot \m_{t+1}$ with $\eta c_l\leq [\eta_{t}]_i\leq \eta c_u$ and $\eta L_F\leq c_l/(2c_u^2)$, we have
\begin{align*}
&F(\w_{t+1})  \leq F(\w_t) +   \frac{ \eta c_u}{2}\|\nabla F(\w_t) - \m_{t+1}\|^2- \frac{\eta c_l}{2}\|\nabla F(\w_t)\|^2  - \frac{\eta c_l}{4}\|\m_{t+1}\|^2.
\end{align*}
\end{lemma}
Combining \eqref{eqn:fin:1} and Lemma \ref{lem:smooth:adam}, and with $(2L_F^2+10C_g^4L_f^2)\eta^2c_u^3/(c_l\beta^2)\leq 1/2$, we get 
we get 
\begin{equation}
\begin{split}
	\E\left[\frac{1}{T}\sum_{t=1}^T\|\nabla F(\w_t)\|^2\right]\leq {} &\frac{2( F(\w_{1})-F(\w_{T+1}))/c_l}{\eta T}+ \frac{c_u(C_F^2+5C_g^2L_f^2G^2)/c_l}{\beta T}\\
	{} & +\beta c_u(5C_g^2L_f^2G^2\frac{m}{B}+2C_g^2C_f^2)/c_l.
\end{split}	
\end{equation}
Thus, by setting $\beta=O(\epsilon^2B/m)$ and $\eta=O(\beta)$, we have 
$\E\left[\frac{1}{T}\sum_{t=1}^T\|\nabla F_{\w}(\w_t)\|^2\right]=O(\epsilon^2)$ for $$T\geq \max \left\{\frac{2(F(\w_{1})-F(\w_{T+1}))}{c_l\eta \epsilon^2},\frac{c_u(C_F^2+5C_g^2L_f^2G^2)}{c_l\beta\epsilon^2}\right\}.$$
\section{COMPARISONS BETWEEN MOAP-V1 AND MOAP-V2}
\begin{table*}[t]
	\caption{Final averaged AP scores on the testing data.}\label{supp:exp:tab:3} 
	\centering
	\label{tab:1}
	\scalebox{0.96}{\begin{tabular}{l|c|c}
			\toprule
		 Method	&phishing&w6a\\
		\hline 
		ADAP&${ 0.981\pm2}${E-}${7}$&${ 0.675\pm1}${E-}${ 4}$\\
		MOAP-V2&$0.978\pm2$E-$6$&$0.596\pm2$E-$3$\\
		MOAP-V1&$0.972\pm4$E-$6$&$0.608\pm4$E-$4$\\
		\bottomrule
	\end{tabular}
}
\vspace*{0.2in}
\end{table*}
\begin{figure}[t]
\centering  
\subfigure[phishing dataset]{
\label{sFig.sub.3}
\includegraphics[width=0.4\textwidth]{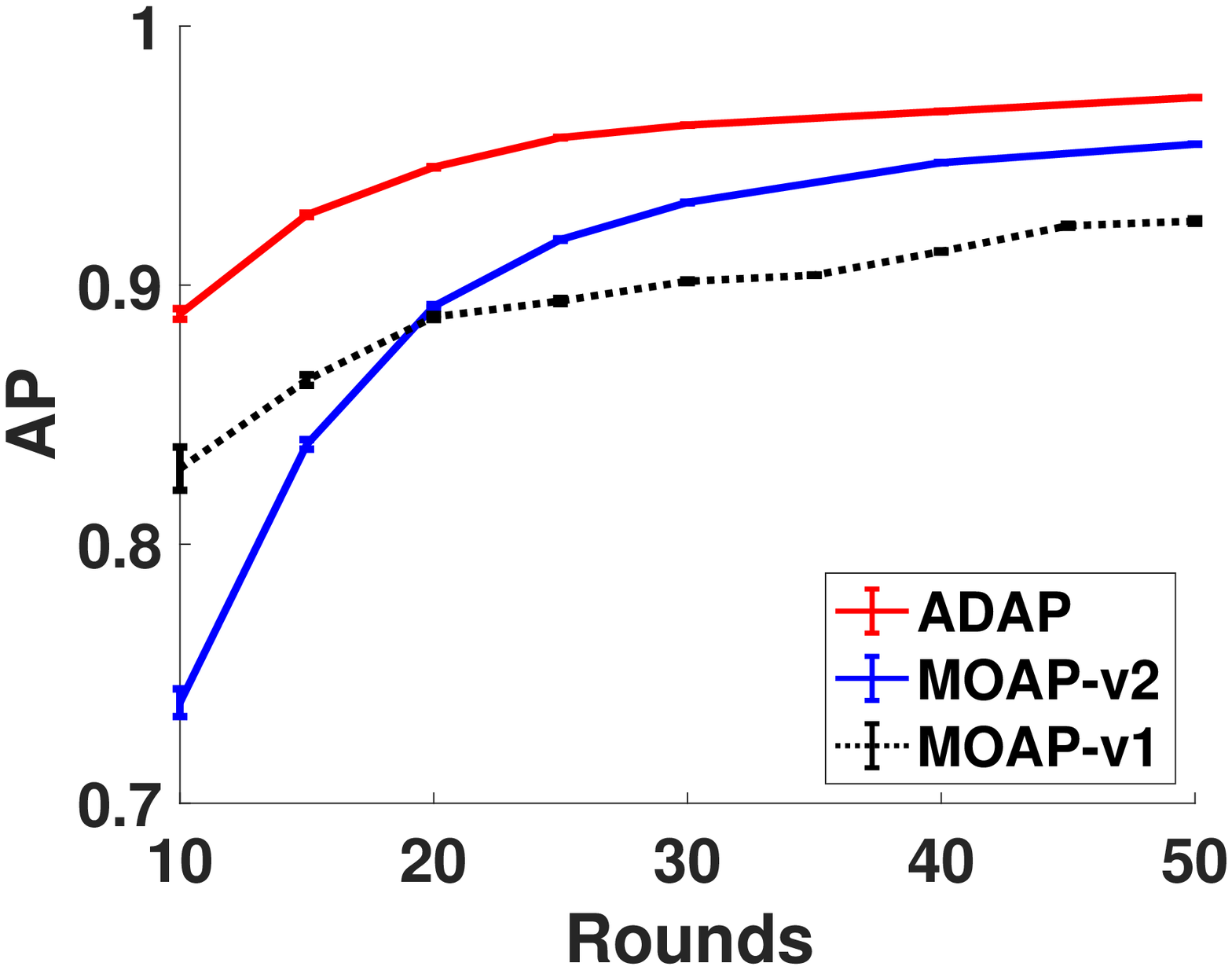}}
\subfigure[w6a dataset]{
\label{sFig.sub.4}
\includegraphics[width=0.4\textwidth]{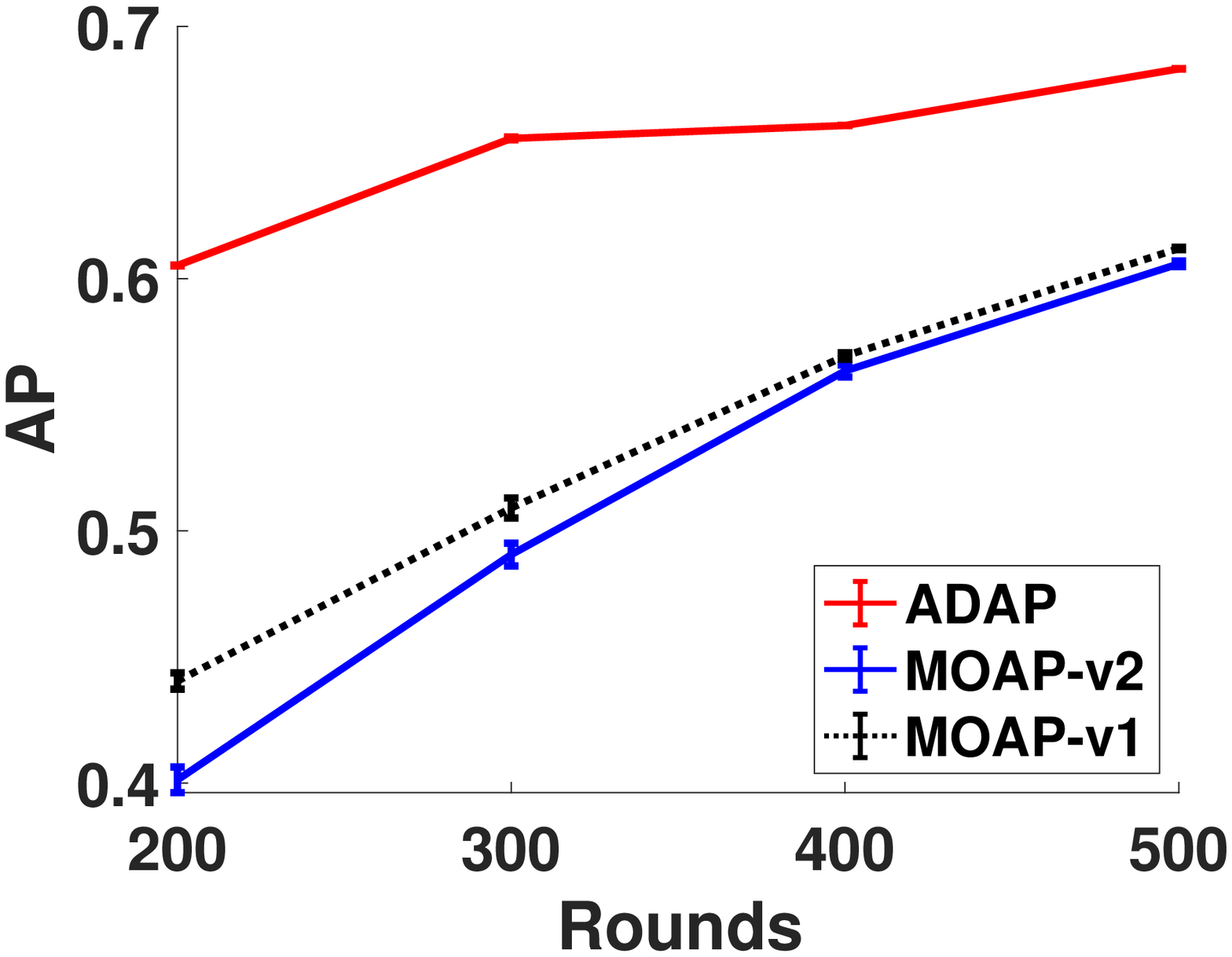}} 
\vspace*{0.2in}
\caption{AP vs \# of rounds  on the training set}
\label{sFig.main}
\end{figure}
\begin{figure}[t]
\centering  
\subfigure[phishing dataset]{
\label{sFig.sub.3}
\includegraphics[width=0.4\textwidth]{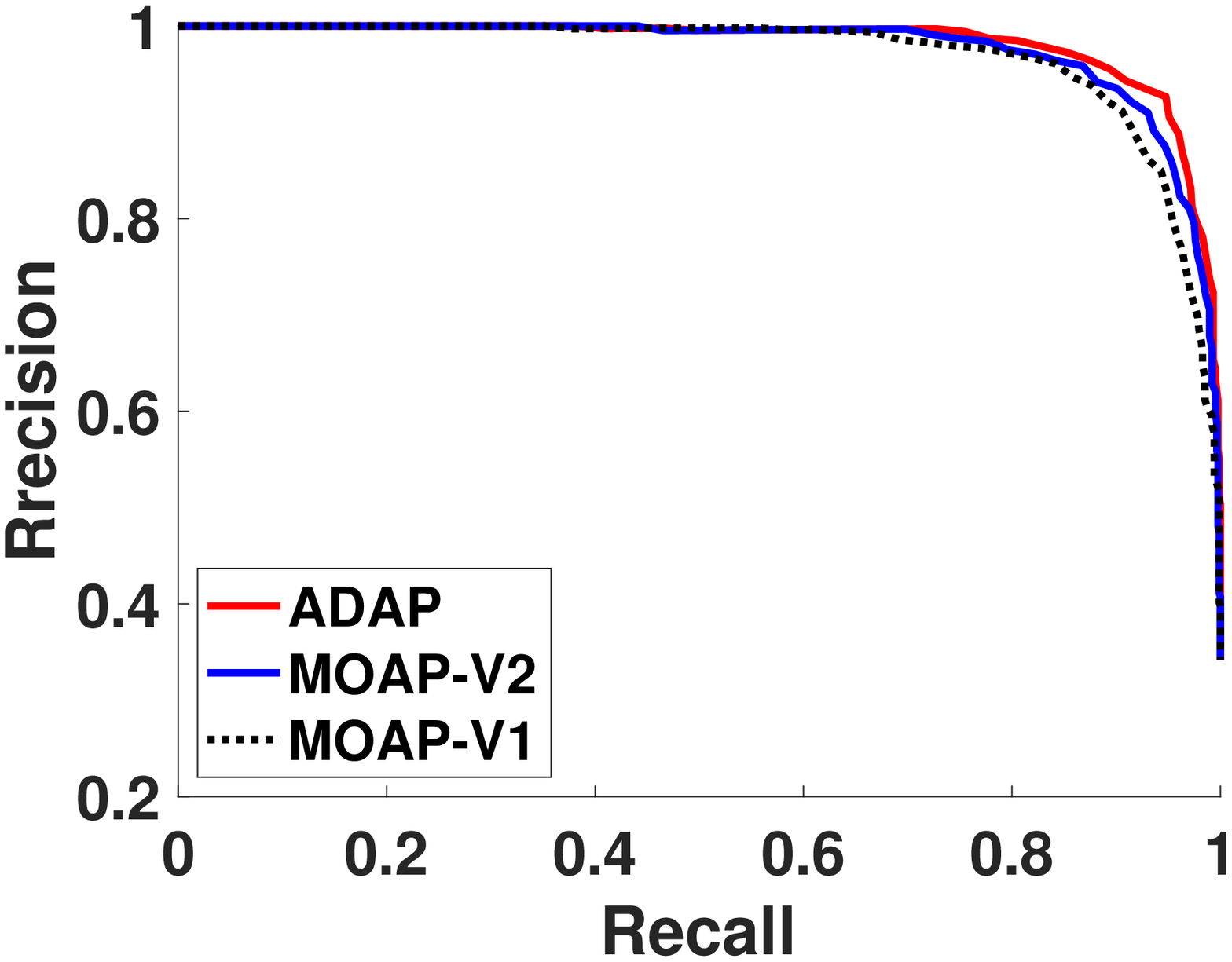}}
\subfigure[w6a dataset]{
\label{sFig.sub.4}
\includegraphics[width=0.4\textwidth]{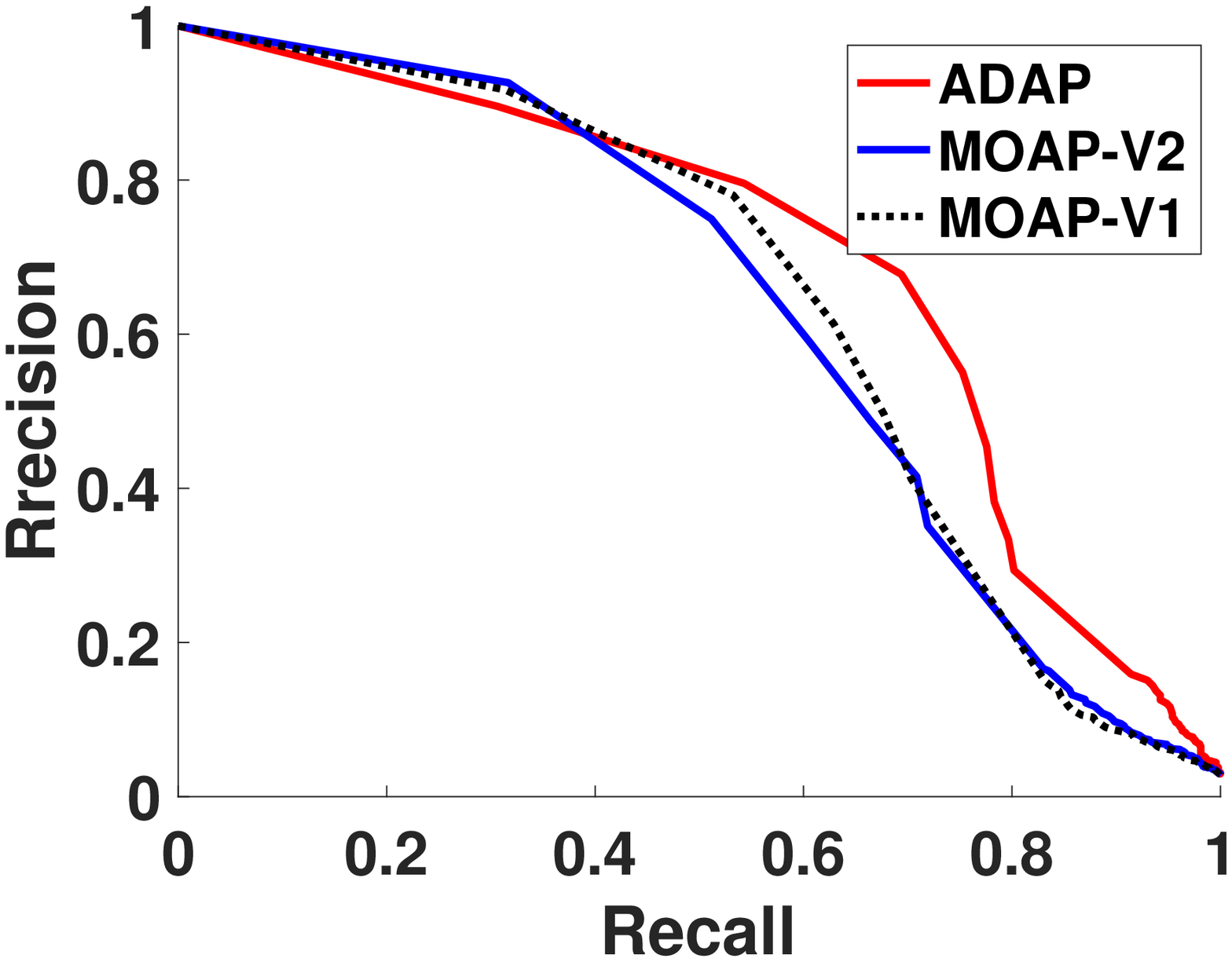}} 
\vspace*{-0.1in}
\caption{Precision-Recall curves of the Final models on the testing set}
\label{sFig.main2}
\end{figure}
In this section, we first discuss the differences between the analysing techniques of MOAP-V1 and MOAP-V2, and then compare their empirical performances on real-world datasets.
\paragraph{Comparison of proofs} We note that, when $B=m$, that is, when $\B_t$ includes all positive data, the proofs of the two algorithms are the same. For $B<m$, the major difference lies in the techniques used for bounding $\sum_{i=1}^T\E_i[\|U_{i+1}-g(\w_{i})\|^2]$, i.e., the variance term of $g(\w)$ estimation. In MOAP-V2, $U_{t+1}$ is the moving average of $\widehat{g}(\w_t)$ (defined in \eqref{eqn:cord}), which is an unbiased estimator of $g(\w_t)$. Thus, we can directly use the variance recursion property (Lemma \ref{lem:variance:recursion}) on $U_{t+1}$ to bound $\sum_{i=1}^T\E_i[\|U_{i+1}-g(\w_{i})\|^2]$ (Eq. \eqref{eqn:appendix:theorem:1:2:3}), and then further combine the property of momentum step (Lemma \ref{Lemma:smooth}) to obtain a tight bound. In contrast, the $U_{t+1}$ in MOAP-V1 is only updated for the sampled data, and thus tracking a biased estimator of $g(\w_t)$. Because of this problem, Lemma \ref{lem:variance:recursion} can not be directly applied on $U_{t+1}$, and we have to bound $\sum_{i=1}^T\E_i[\|U_{i+1}-g(\w_{i})\|^2]$ coordinate-wisely as in \eqref{eqn:th2:356}, using the fact that $\tilde{g}_i(\w_t)$ is an biased estimator of ${g}_i(\w_t)$. However, the dependent issue makes the advantages of the momentum step in Lemma 3 can not be fully exploited, which leads to a worse convergence rate. Whether we can overcome the dependant problem and directly improve the convergence rate of MOAP-V1 is still an open question.

\paragraph{Experiments}
Next, we compare the performances of MOAP-V1 and MOAP-V2 on two real-world datasets, that is, \emph{w6a} and \emph{phishing}. The parameter configurations are the same as the experiments in Section \ref{section:exp}. The convergence curves of AP on training examples are reported in Figure \ref{sFig.main}, and the final AP scores on the testing data are shown in Table \ref{exp:tab:3}. We also plot the Precision-Recall curves of the final models on testing data in Figure~\ref{sFig.main2}. From these results, it can be seen that, although MOAP-V1 currently suffers worse theoretical guarantees, it achieves comparable empirical results as MOAP-V2. 
\end{document}